\documentclass[english]{article}

\usepackage{geometry}
\usepackage[T1]{fontenc}
\usepackage[latin9]{inputenc}
\usepackage{bm}
\usepackage{amsmath,mathtools}
\usepackage{amssymb}
\usepackage[unicode=true,
 bookmarks=false,
 breaklinks=false,pdfborder={0 0 1},colorlinks=false]
 {hyperref}
\hypersetup{
 colorlinks,citecolor=blue,filecolor=blue,linkcolor=blue,urlcolor=blue}

\makeatletter
\usepackage{amsthm,enumitem}
\usepackage{cite}  
\usepackage{comment}
\usepackage{natbib}
\usepackage{booktabs}
\usepackage{graphicx}

\usepackage[linesnumbered,ruled,vlined]{algorithm2e}
\usepackage{algorithmic}

\usepackage{float}
\usepackage{multirow}
\usepackage{dsfont}
\usepackage{tcolorbox}

\usepackage{color}
\definecolor{yxc}{RGB}{255,0,0}
\definecolor{yjc}{RGB}{125,0,0}
\definecolor{ytw}{RGB}{255,69,0}
\definecolor{gen}{RGB}{0,0,200}

\allowdisplaybreaks

\DeclareMathOperator{\ind}{\mathds{1}}  

\newcommand{\defn}{\coloneqq}

\newcommand{\cS}{\mathcal{S}}
\newcommand{\cA}{\mathcal{A}}
\newcommand{\cB}{\mathcal{B}}
\newcommand{\ak}{a_{k,s}}

\newcommand{\mycrosstwo}{,}

\newcommand{\mymid}{\,|\,}

\newcommand{\cb}{c_{\mathsf{b}}}

\newcommand{\pihat}{\widehat{\pi}}
\newcommand{\pihatstar}{\widetilde{\pi}^{\star}}
\newcommand{\calpha}{c_{\alpha}}
\newcommand{\NEgap}{\mathsf{gap}}
\newcommand{\myalg}{\textsf{Q-FTRL}}

\theoremstyle{plain} \newtheorem{lemma}{\textbf{Lemma}}\newtheorem{theorem}{\textbf{Theorem}}

\theoremstyle{remark}\newtheorem{remark}{\textbf{Remark}}

\makeatother

\begin{document}

\title{Minimax-Optimal Multi-Agent RL in Markov Games \\ With a Generative Model}

\author{Gen Li\thanks{Department of Statistics and Data Science, Wharton School, University of Pennsylvania, Philadelphia, PA 19104, USA.} \\
UPenn    \\
	\and
	Yuejie Chi\thanks{Department of Electrical and Computer Engineering, Carnegie Mellon University, Pittsburgh, PA 15213, USA.}\\
	CMU\\
	\and
	Yuting Wei\footnotemark[1] \\
 UPenn  \\
	\and
	Yuxin Chen\footnotemark[1]  \\
	UPenn\\
	}
\date{August 2022; ~~ Revised: October 2022}

\maketitle

\begin{abstract}
	
This paper studies multi-agent reinforcement learning in Markov games,
with the goal of learning Nash equilibria or coarse correlated equilibria (CCE) sample-optimally. 
All prior results suffer from at least one of the two obstacles: the curse of multiple agents and the barrier of long horizon, regardless of the sampling protocol in use.  
We take a step towards settling this problem, assuming access to a flexible sampling mechanism: the generative model.  
Focusing on non-stationary finite-horizon Markov games,  
we develop a fast learning algorithm called \myalg~and an adaptive sampling scheme that leverage the optimism principle in online adversarial learning (particularly the Follow-the-Regularized-Leader (FTRL) method). 
Our algorithm learns an $\varepsilon$-approximate CCE in a general-sum Markov game using  
$$ \widetilde{O}\bigg( \frac{H^4 S \sum_{i=1}^m A_i}{\varepsilon^2} \bigg) $$ 
samples, where $m$ is the number of players, $S$ indicates the number of states, $H$ is the horizon, and $A_i$ denotes the number of actions for the $i$-th player. 
This is minimax-optimal (up to log factor) when the number of players is fixed.  
When applied to two-player zero-sum Markov games, our algorithm provably finds an $\varepsilon$-approximate Nash equilibrium with minimal samples. 
Along the way, we derive a refined regret bound for FTRL that makes explicit the role of variance-type quantities, which might be of independent interest.

\end{abstract}






\medskip
\noindent \textbf{Keywords:} Markov games, sample complexity, Nash equilibrium, coarse correlated equilibrium, adversarial learning, Follow-the-Regularized-Leader 

\setcounter{tocdepth}{2}
\tableofcontents

\section{Introduction}

The thriving field of multi-agent reinforcement learning (MARL) studies how a group of interacting agents make decisions autonomously 
in a shared dynamic environment \citep{zhang2021multi}. 
The recent developments in game playing \citep{vinyals2019grandmaster,brown2019superhuman}, 
self-driving vehicles \citep{shalev2016safe}, and multi-robot control \citep{matignon2012coordinated} 
are prime examples of MARL in action. 
In practice, there is no shortage of situations where the agents involved have conflict of interest, 
and they have to act competitively in order to promote their own benefits (possibly at the expense of one another).   
Scenarios of this kind are frequently modeled via Markov games (MGs) \citep{shapley1953stochastic,littman1994markov}, 
a framework that has been a fruitful playground to formalize and stimulate the studies of competitive MARL.

In view of the irreconcilable competition between individual players, 
solutions of competitive MARL normally take the form of certain equilibrium  strategy profiles, 
which are perhaps best epitomized by the concept of Nash equilibrium (NE) \citep{nash1950equilibrium}.   
In a Nash equilibrium, no gain can be realized through a unilateral change --- assuming no coordination between players ---  
and hence no player has incentives to deviate from her current strategy/policy. 
A myriad of research has been conducted surrounding NE, which spans various aspects like existence, learnability, computational hardness, and algorithm design, among others \citep{shapley1953stochastic,daskalakis2013complexity,chen2015well,rubinstein2016settling,perolat2015approximate,daskalakis2020independent,littman1994markov,hansen2013strategy,ozdaglar2021independent,jin2022complexity}. 
Given that finding NE is notoriously expensive in general (except for special cases like two-player zero-sum MGs) \citep{daskalakis2013complexity,daskalakis2009complexity}, 
several more tractable solution concepts have emerged in the studies of game theory and MARL, 
a prominent example being the coarse correlated equilibirum (CCE) \citep{moulin1978strategically}. 
A key compromise made in the CCE is that it permits the players to act in an coordinated fashion, which contrasts sharply with the absence of coordination in the definition of NE.

One critical challenge impacting modern MARL applications is data efficiency. 
The players involved often have minimal knowledge about how the environment responds to their actions, 
and have to learn the dynamics and preferable actions by probing the unknown environment. 
For MARL to expand into applications with enormous dimensionality and long planning horizon, 
the learning algorithms must manage to make efficient use of the collected data.  
Nevertheless, how to learn NE and/or CCE with optimal sample complexity remains by and large unsettled 
even when it comes to the most basic setting: two-player zero-sum Markov games, as we shall discuss below.

\paragraph{Example: inadequacy in learning two-player zero-sum Markov games.}

To facilitate concrete comparisons, let us review two representative algorithms aimed at learning NE in two-player zero-sum MGs. 
These algorithms have been studied  under two drastically different sampling protocols, 
and we shall discuss the shortfalls of the cutting-edge sample complexity results. 
In a two-player zero-sum MG, we denote by $S$ the number of states and $H$ the horizon or effective horizon, whereas $A_1$ (resp.~$A_2$) denotes the number of actions for the max-player (resp.~min-player).

\begin{itemize}
	\item {\em Model-based methods under either a generative model or online exploration.} Assuming access to a generative model (so that one can sample arbitrary state-action tuples),  
		\citet{zhang2020marl} investigated a natural model-based algorithm, which performs planning (e.g., value iteration) on an empirical MG derived from samples produced non-adaptively by the generative model. 
	 Focusing on {\em stationary} discounted infinite-horizon MGs, their algorithm finds an $\varepsilon$-approximate NE with no more than
\begin{equation}
	\widetilde{O}\bigg( \frac{H^3 SA_1A_2}{\varepsilon^2} \bigg) ~~\text{samples}. 
	\label{eq:Zhang-plug-in-result}
\end{equation}
In parallel, \citet{liu2021sharp} studied {\em non-stationary} finite-horizon MGs with online exploration, and obtained similar sample complexity bounds, i.e.,  
\begin{equation}
	\widetilde{O}\bigg( \frac{H^4 SA_1A_2}{\varepsilon^2} \bigg) ~~\text{samples}
	\qquad \text{or} \qquad 
	\widetilde{O}\bigg( \frac{H^3 SA_1A_2}{\varepsilon^2} \bigg) ~~\text{episodes} 
	\label{eq:Liu-MB-result}
\end{equation}
for  learning an $\varepsilon$-approximate NE. 
While these bounds achieve minimax-optimal dependency on the horizon $H$, 
		a major drawback emerges --- commonly referred to as the curse of multiple agents; namely, these results scale proportionally with the total number of {\em joint actions} (i.e., $\prod_{1\leq i\leq 2}A_i$), a quantity that blows up exponentially with the number of players.

\item {\em V-learning for online exploration settings.} Focusing on online exploration settings, 
	\citet{bai2020near,jin2021v} proposed an algorithm called V-learning that leverages the advances in online adversarial learning (e.g., adversarial bandits) to circumvent the curse of multiple agents. This algorithm provably yields an $\varepsilon$-approximate NE in non-stationary finite-horizon MGs using
\begin{equation}
	\widetilde{O}\bigg( \frac{H^6 S(A_1+A_2)}{\varepsilon^2} \bigg) ~~\text{samples}
	\qquad \text{or} \qquad 
	\widetilde{O}\bigg( \frac{H^5 S(A_1+A_2)}{\varepsilon^2} \bigg) ~~\text{episodes}, 
	\label{eq:Jin-v-learning-result}
\end{equation}
which effectively brings down the sample size scaling \eqref{eq:Liu-MB-result} from $A_1A_2$ (i.e., the number of joint actions) to $A_1+A_2$ (i.e., the sum of individual actions). 
It is worth pointing out, however, that this theory appears sub-optimal in terms of the horizon dependency,  as it is a factor of $H^2$ above the minimax lower bound.
\end{itemize}



\paragraph{Key issues and our main contributions.}
While the above summary focuses on two-player zero-sum MGs, 
it unveils a fundamental issue surrounding the sample efficiency of learning equilibria; 
that is, all existing results in this front --- irrespective of the sampling mechanism in use --- fall short of overcoming at least one of the two major hurdles:  (i) the {\em curse of multiple agents}, and (ii) the {\em barrier of long horizon}.  
A natural question to pose is:
%
%
\begin{center}
{\em 
	{\bf Question: }can we learn a Nash equilibrium in a two-player zero-sum Markov game \\ in a sample-optimal and computation-efficient fashion?}
\end{center}
\noindent To settle this favorably, both of the above hurdles need to be crossed simultaneously. 
Moving beyond two-player zero-sum MGs, it is not surprising to see that general-sum multi-player MGs have to grapple with the aforementioned two hurdles as well. 
Thus, the following question also comes into mind when learning CCE (a compromise due to the general intractability of learning NE):  
\begin{center}
{\em 
	{\bf Question: }can we learn a coarse correlated equilibrium in a multi-player general-sum Markov game \\ in a sample-optimal and computation-efficient fashion?}
\end{center}
Note that these questions remain open regardless of the sampling scheme in use.

\noindent


This paper takes a first step towards solving the problem by assuming access  to the most flexible sampling protocol: the generative model (or simulator). 
In stark contrast to the single-agent case where uniform sampling of all state-action pairs suffices \citep{azar2013minimax,li2020breaking},  
the multi-agent scenario requires one to take samples intelligently and adaptively, 
a crucial step to avoid inefficient use of data (otherwise one cannot hope to break the curse of multiple agents). 
With the aim of computing an $\varepsilon$-approximate equilibrium in a {\em non-stationary} finite-horizon MG, 
we come up with a computationally efficient learning algorithm (accompanied by an adaptive sampling strategy) that accomplishes this goal with no more than
\begin{align}
	\begin{cases}
		\widetilde{O}\Big( \frac{H^4 S(A_1+A_2)}{\varepsilon^2} \Big) ~~\text{samples} \qquad & \text{(learning }\varepsilon\text{-NE in two-player zero-sum MGs)}\\
		\widetilde{O}\Big( \frac{H^4 S\big(\sum_{i=1}^m A_i \big)}{\varepsilon^2} \Big) ~~\text{samples}
		& \text{(learning }\varepsilon\text{-CCE in multi-player general-sum MGs)}
	\end{cases} 
\end{align}
drawn from the generative model. Encouragingly, this sample complexity bound matches the minimax lower limit (up to some log factor) as long as the number of players $m\geq 2$ is a fixed constant or grows only logarithmically fast. 
Our sample complexity theory is valid for the full $\varepsilon$-range (i.e., any $\varepsilon \in (0,H]$); 
this unveils that no burn-in cost whatsoever is needed for our algorithm to achieve sample optimality, which lends itself well to sample-hungry applications.

The proposed algorithm is inspired by two key algorithmic ideas in RL and bandit literature: (i) optimism in the face of uncertainty (by leveraging upper confidence bounds (UCBs) in value estimation), 
and (ii) online and adversarial learning (particularly the Follow-the-Regularized-Leader (FTRL) algorithm). 
Note that the optimal design of bonus terms --- typically based on certain data-driven variance estimates --- is substantially more challenging than the single-agent case,   
 as it requires intricate adaptation in response to the policy changes of one another as well as compatibility with the FTRL dynamics.  
Two points are worth emphasizing (which will be made precise later on): 
\begin{itemize}
	\item The efficacy of FTRL in breaking the curse of multiple agents has been illustrated in \citet{jin2021v,song2021can,mao2022provably}. 
		To improve horizon dependency, one needs to exploit connections between the performance of FTRL and certain variances. 
		Towards this, we develop a refined regret bound for FTRL that unveils the role of variance-style quantities, which was previously unavailable. 

	
	\item The bonus terms entail Bernstein-style variance estimates that mimic the variance-style quantities appearing in our refined FTRL regret bounds,  and are carefully chosen so as to ensure certain decomposability over steps. 
		This is crucial in optimizing the horizon dependency. 
\end{itemize}
Additionally, the policy returned by our algorithm is Markovian (i.e., the action selection probability depends only on the current state $s$ and step $h$), 
and the algorithm can be carried out in a decentralized manner without the need of directly observing the opponents' actions.

\paragraph{Other related works.}

%
Let us discuss in passing additional prior works on learning equilibrium solutions in MARL, which 
have attracted an explosion of interest in recent years. 
While the Nash equilibrium is arguably the most compelling solution concept in Markov games, 
the finite-sample/finite-time studies of NE learning concentrate primarily on two-player zero-sum MGs (e.g., \citet{bai2020provable,chen2022almost,mao2022provably,wei2017online,tian2021online,cui2022offline,cui2022provably,zhong2022pessimistic,jia2019feature,yang2022t,yan2022model,dou2022gap}), 
mainly because computing NEs becomes, for the most part, computationally infeasible (i.e., PPAD-complete) when going beyond two-player zero-sum MGs \citep{daskalakis2013complexity,daskalakis2009complexity}.  
Roughly speaking, previous NE-finding algorithms for two-player zero-sum Markov games can be categorized into model-based algorithms \citep{perolat2015approximate,zhang2020marl,liu2021sharp}, value-based algorithms \citep{bai2020provable,bai2020near,xie2020learning,sayin2021decentralized,jin2021v,chen2021almost}, and policy-based algorithms \citep{cen2021fast,daskalakis2020independent,wei2021last,zhao2021provably,chen2021sample,zhang2022policy,cen2022faster}. 
In particular, \citet{bai2020near,jin2021v} developed the first algorithms to beat the curse of multiple agents in two-player zero-sum MGs, 
while \citet{jin2021v,daskalakis2022complexity,mao2022provably,song2021can} further demonstrated how to accomplish the same goal 
when learning other computationally tractable solution concepts (e.g., coarse correlated equilibria) in general-sum multi-player Markov games.  
The recent works \citet{cui2022offline, cui2022provably,yan2022model} studied how to alleviate the sample size scaling with the number of agents in the presence of offline data, 
with \citet{cui2022provably} providing a sample-efficient algorithm that also learns NEs in multi-agent Markov games (despite computational intractability). 

We shall also briefly remark on the prior works that concern RL with a generative model. 
While there are multiple sampling mechanisms (e.g., online exploratory sampling, offline data)  that bear practical relevance, 
the generative model (or simulator) serves as an idealistic sampling protocol that has received much recent attention, 
covering the design of various model-based, model-free and policy-based algorithms 
\citep{kearns2002sparse, wang2021sample, agarwal2020model,azar2013minimax,li2020breaking, jin2021towards, sidford2018near, wainwright2019stochastic, wainwright2019variance, li2021q, kakade2003sample, sidford2018variance,pananjady2020instance, khamaru2020temporal,chen2020finite,even2003learning,beck2012error,wei2021last,vaswani2022near,zanette2019almost,weisz2021exponential,mou2020linear,yang2019sample,zanette2020provably,du2020good}.   
In single-agent RL, the model-based approach has been shown to be minimax-optimal for the entire $\varepsilon$-range \citep{li2020breaking, agarwal2020model,azar2013minimax}. 
When it comes to multi-agent RL, sample-efficient solutions with a generative model have been proposed in the recent works 
\citep{sidford2020solving,cui2021minimax,zhang2020marl}, although a provably sample-optimal strategy was previously unavailable.

\paragraph{Paper organization and notation.}

The rest of the paper is organized as follows. Section~\ref{sec:background} introduces the background of Markov games, the preliminaries of the solution concepts of NE and CCE, and formulates the sampling protocol. 
The proposed learning algorithm and the sampling strategy are described in Section~\ref{sec:algorithm}, 
with the theoretical guarantees provided in Section~\ref{sec:main-results}. 
Section~\ref{sec:FTRL} takes a detour to develop our refined regret bound for FTRL, which plays a crucial role in our main sample complexity analysis in Section~\ref{sec:analysis}. 
Proof details (particularly those for auxiliary lemmas) are postponed to the appendix.

Let us also gather several convenient notation that shall be used multiple times. 
For any positive integer $n$, we write $[n] \coloneqq \{1,\cdots,n\}$. We shall abuse notation and let $1$ and $0$ denote the all-one vector and the all-zero vector, respectively. 
For a sequence $\{\alpha_k\}_{k\geq 1}\subseteq (0,1]$, we define  
\begin{equation}
\alpha_{i}^{k}\coloneqq\begin{cases}
\alpha_{i}\prod_{j=i+1}^{k}(1-\alpha_{j}), & \text{if }0<i<k\\
\alpha_{k}, & \text{if }i=k
\end{cases}\label{def:alpha-i-k}
\end{equation}
for any $1\leq i\leq k$.  
For a given vector $x\in \mathbb{R}^{SA}$ (resp.~$y\in \mathbb{R}^{SAB}$), we denote by $x(s,a)$ (resp.~$y(s,a,b)$) the entry of $x$ (resp.~$y$) associated with
the state-action combination $(s,a)$ (resp.~$(s,a,b)$), as long as it is clear from the context. 
Next, consider any two vectors $a=[a_i]_{1\leq i\leq n}$ and $b=[b_i]_{1\leq i\leq n}$. 
We use $a\leq b$ (resp.~$a\geq b$) to indicate that $a_i\geq b_i$ (resp.~$a_i\leq b_i$) holds for all $i$; 
 we allow scalar functions to take vector-valued arguments in order to denote entrywise operations 
(e.g., $a^2=[a_i^2]_{1\leq i\leq n}$ and $a^4=[a_i^4]_{1\leq i\leq n}$); and we denote by $a \circ b =[a_ib_i]_{1\leq i\leq n}$ the Hadamard product.  
For a finite set $\cA=\{1,\cdots,A\}$, we denote by $\Delta(\cA)=\{x\in \mathbb{R}^A\mymid \sum_i x_i = 1; x \geq 0\}$ the probability simplex over $\cA$. 
For any function $f$ with domain $\cA$ (or $\cB$), we adopt the convenient notation
%
\begin{align}
	\mathbb{E}_{\pi}[f] &\defn \sum\nolimits_{a} \pi(a)f(a)
	\qquad\text{and}\qquad
	\mathsf{Var}_{\pi}(f) \defn \sum\nolimits_{a} \pi(a)\big( f(a) - \mathbb{E}_{\pi}[f]\big)^2. 
	\label{eq:notation-Epi-Varpi} 
\end{align}
%

\section{Background and models}
\label{sec:background}

In this section, we introduce the basics for  Markov games,  
as well as the solution concepts of Nash equilibrium and coarse correlated equilibrium.

\paragraph{Markov games.} 
A non-stationary finite-horizon {\em multi-player general-sum Markov game}, denoted by $\mathcal{MG}=\big\{ \cS, \{\cA_i\}_{1 \le i \le m}, H, P, r \big\}$, 
involves $m$ players competing against each other, and consists of several key elements to be formalized below. Recall that $\Delta(\cS)$ represents the probability simplex over the set $\cS$. 
\begin{itemize}
	\item $\cS=\{1,\cdots,S\}$ is the state space of the shared environment, which comprises $S$ different states. 
	\item For each $1\leq i\leq m$,  let $\mathcal{A}_i=\{1,\cdots, A_i\}$ represent the action space of the $i$-th player, which contains $A_i$ different actions. 
		Here and below, we denote
		\begin{equation}
			\mathcal{A} \coloneqq \mathcal{A}_1 \times \cdots \times \mathcal{A}_m
			\qquad \text{and} \qquad
			\mathcal{A}_{-i} \coloneqq \prod_{j:j\neq i}\mathcal{A}_j ~~~ (1\leq i\leq m). 
		\end{equation}
		Throughout the paper, we shall often use the boldface letter $\bm{a} \in \cA$ (resp.~$\bm{a}_{-i} \in \cA_{-i}$) to denote a joint action profile of all players (resp.~a joint action profile excluding the $i$-th player's action).  
		
	\item $H$ stands for the horizon length of the Markov game. 

	\item $P=\{P_h\}_{1\leq h \leq H}$ --- with $P_h: \cS \times \cA \rightarrow \Delta(\cS)$ --- denotes the probability transition kernel of $\mathcal{MG}$. 
		Namely, for any $(s, \bm{a}, h,s')\in \cS\times \cA \times [H] \times \cS$, we let $P_h(s'\mymid s,\bm{a})$ indicate the probability of $\mathcal{MG}$ transitioning from state $s$ to state $s'$ at step $h$ when the joint action profile taken by the players is $\bm{a}$. 

	\item $r=\{r_{i,h}\}_{1\leq h \leq H, 1\leq i\leq m}$ --- with $r_{i,h}: \cS\times \cA \rightarrow [0,1]$ --- represents the (deterministic) reward function. 
		Namely, for any $(s,\bm{a},h)\in \cS\times \cA \times [H]$, $r_{i,h}(s,\bm{a})$ stands for the immediate reward the $i$-th player gains in state $s$ at step $h$, if the joint action profile is $\bm{a}$. 
		Here and throughout, we assume normalized rewards in the sense that $r_{i,h}(s,\bm{a})\in [0,1]$ for any $(s,\bm{a},h,i)\in \cS\times \cA \times[H] \times [m]$.

\end{itemize}

\noindent 
As an important special case, a {\em two-player zero-sum Markov game} --- denoted by $\mathcal{MG}=\big\{ \cS, \{\cA_1, \cA_2\}, H, P, r \big\}$ --- 
satisfies $r_{2,h}=-r_{1,h}$ for all $h\in [H]$. 
Following the convention, we assume that $r_{1,h}\geq 0$ for all $h\in [H]$,\footnote{The careful reader might immediately note that $r_{2,h}\leq 0$, thus falling outside our assumed range for the reward function. 
This, however, can be easily addressed by enforcing a positive global shift to $r_{2,h}$ without changing the learning process.} 
and refer to the first (resp.~second) player as the max-player (resp.~the min-player).




\paragraph{Markov policies.} 
This paper focuses on the class of Markov policies, such that the action selection strategies of the players are determined by the current state $s$ and the step number $h$, 
without depending on previously visited states.  
To begin with, 
let $\pi_i=\{\pi_{i,h}\}_{1\leq h\leq H}$ represent the policy of the $i$-th player.  
Here, $\pi_{i,h}(\cdot \mymid s) \in \Delta(\cA_i)$ for any $(s,h)\in \cS \times [H]$,  
where $\pi_{i,h}(a \mymid s)$ indicates the probability of the $i$-th player selecting action $a$ in state $s$ at step $h$. 
The joint Markov policy can be defined analogously: 
we let $\pi= (\pi_1,\ldots, \pi_m): \cS  \times [H] \rightarrow \Delta(\cA)$  represent a joint Markov policy of all players, 
where the joint actions of all players in state $s$ and step $h$ are chosen according to the distribution specified by $\pi_h(\cdot \mymid s) = (\pi_{1,h},\ldots, \pi_{m,h})(\cdot\mymid s) \in \Delta(\cA)$. 
For any given joint policy $\pi$, we employ $\pi_{-i}$ to represent the policies of all but the $i$-th player, 
and let $\pi_{-i,h}$ denote the policies of all but the $i$-th player at step $h$. 
All policies are assumed throughout to be Markovian, except our brief remarks on non-Markovian policies in Section~\ref{sec:main-results}. 

Additionally, a joint policy $\pi$ is said to be a {\em product policy} if $\pi_{1},\ldots,\pi_m$ are executed in a statistically independent fashion (namely, under policy $\pi$ the players take actions independently), 
and we shall adopt the notation $\pi=\pi_1\times \cdots \times \pi_m$ to indicate that $\pi$ is a product policy.

\paragraph{Value functions.} 
Consider a Markovian trajectory $\{(s_h,\bm{a}_h)\}_{1\leq h\leq H}$, where $s_h\in \cS$ is the state at step $h$ and $\bm{a}_h\in \cA$ is the joint action profile at step $h$. 
For any given joint policy $\pi$ and any step $h\in[H]$, 
we define the value function $V_{i,h}^{\pi}: \cS \rightarrow \mathbb{R}$ of the $i$-th player under policy $\pi$ as follows: 
\begin{equation}
	V_{i,h}^{\pi}(s)\coloneqq\mathbb{E}\left[\sum_{t=h}^{H}r_{i,t}\big(s_{t}, \bm{a}_{t}\big)\mid s_{h}=s\right],
\qquad \forall s\in \cS, 
\label{eq:value-function-defn}
\end{equation}
where the expectation is taken over the Markovian trajectory $\{(s_h,\bm{a}_h)\}$ with the $m$ players jointly executing policy $\pi$; 
that is, conditional on $s_h$, we draw $\bm{a}_h\sim \pi_h(\cdot \mymid s_h)$  and then $s_{h+1}\sim P_h(\cdot \mymid s_h, \bm{a}_h)$. 

In addition, consider the case where (i) all but the $i$-th player executes the joint policy $\pi_{-i}$ and (ii) 
the $i$-th player executes policy $\pi'_i$ {\em independently} from the other players;    
we shall denote by $V_{i,h}^{\pi'_i \times \pi_{-i}}$ the resulting value function under this joint policy $\pi'_i \times \pi_{-i}$. 
By optimizing over all $\pi'_i$, we can further define 
\begin{equation}
	\label{eq:defn-optimal-V}
	V_{i,h}^{\star,\pi_{-i}}(s)\coloneqq\max_{\pi'_i: \cS\times [H] \rightarrow \Delta(\mathcal{A}_i)} V_{i,h}^{\pi'_i \times \pi_{-i}}(s) ,
	\qquad \forall (s,h,i) \in \cS \times [H] \times [m].  
\end{equation}
It is known that there exists at least one policy, denoted by $\pi_i^{\star}\big( \pi_{-i} \big): \cS\times [H] \rightarrow \Delta(\mathcal{A}_i)$ and commonly referred to as the {\em best-response policy}, 
that can simultaneously attain $V_{i,h}^{\star,\pi_{-i}}(s)$ for all $h\in [H]$ and all $s\in \cS$. 
It is worth emphasizing that the best-response policy $\pi_i^{\star}\big( \pi_{-i} \big)$ 
is the best among all policies of the $i$-th player executed independently of $\pi_{-i}$. 
Furthermore, if we freeze $\pi_{-i}$, then the Bellman optimality condition for the $i$-th player can be expressed as \citep{bertsekas2017dynamic}
\begin{equation}
	V_{i,h}^{\star,\pi_{-i}}(s)
	=\max_{a_i\in \cA_i}\left\{ \mathop{\mathbb{E}}\limits _{\bm{a}_{-i}\sim \pi_{-i,h}(\cdot | s)}\left[r_{i,h}(s,\bm{a})+\Big\langle P_h(\cdot\mymid s, \bm{a}),\,V_{i,h+1}^{\star,\pi_{-i}}\Big\rangle\right]\right\} ,
	\quad \forall (s,h,i) \in \cS \times [H] \times [m],
	\label{eq:Bellman-equation-max}
\end{equation}
where the joint action profile $\bm{a}$ is composed of $a_i$ for the $i$-th player and $\bm{a}_{-i}$ for the remaining ones.

\paragraph{Equilibria of Markov games.}
In a multi-agent Markov game, each player wishes to maximize its own value function. 
Due to the competing objectives,  finding some sorts of equilibria --- e.g., the Nash equilibrium \citep{nash1951non} and the coarse correlated equilibrium \citep{moulin1978strategically,aumann1987correlated} --- becomes a central topic in the studies of Markov games. Let us introduce these solution concepts below.  
\begin{itemize}
	\item {\em Nash equilibrium.} A product policy $\pi=\pi_1\times \cdots \times \pi_m$ is 
said to be a {\em (mixed-strategy) Nash equilibrium} of $\mathcal{MG}$ if the following holds: 
\begin{equation}
	V_{i, 1}^{\pi}(s)=V_{i, 1}^{\star,\pi_{-i}}(s),
	\qquad\text{for all }(s,i)\in\cS\times [m].
	\label{eq:defn-Nash-E}
\end{equation}
In other words, conditional on the opponents' current policy and the assumption that all players take actions {\em independently}, 
no player can harvest any gain by unilaterally deviating from its current policy.


	\item {\em Coarse correlated equilibrium.}  A joint policy $\pi $ is said to be a coarse correlated equilibrium of $\mathcal{MG}$ if
\begin{equation}
	V_{i, 1}^{\pi}(s) \geq V_{i, 1}^{\star,\pi_{-i}}(s),
	\qquad\text{for all }(s,i)\in\cS\times [m].
	\label{eq:defn-CCE}
\end{equation}
While a CCE also ensures that no unilateral deviation (performed independently from others) is beneficial, 
its key distinction from the definition of NE lies in the fact that it allows the policy to be correlated across the players. Any NE of $\mathcal{MG}$ is, self-evidently, also a CCE. 


\end{itemize}
%

%
%

%
%
%

\noindent In practice, it might be challenging to compute an ``exact'' equilibrium, and instead one would seek to find approximate solutions. 
Towards this end, we find it helpful to define the sub-optimality gap of a policy $\pi$ as follows (measured in an $\ell_{\infty}$-based manner)
\begin{subequations}
\label{eq:defn-NE-gap}
\begin{equation}
	\NEgap(\pi) \coloneqq \max_{s\in\cS} \, \NEgap(\pi;s),
	\label{eq:defn-NE-gap-final}
\end{equation}
where
\begin{equation}
	\NEgap(\pi;s) \, \coloneqq \max_{1\leq i\leq m}\Big\{ V_{i, 1}^{\star,\pi_{-i}}(s)-V_{i, 1}^{\pi}(s)\Big\} .
	\label{eq:defn-NE-gap-s}
\end{equation}
\end{subequations}
With this sub-optimality measure in place, a {\em product} policy $\pi=\pi_1\times \cdots \times \pi_m$ is said to be an $\varepsilon$-approximate NE --- or more concisely, $\varepsilon$-Nash --- if the resultant sub-optimality gap obeys $\NEgap(\pi) \leq \varepsilon$. 
Similarly, a joint (and possibly correlated) policy  $\pi $ is said to be an $\varepsilon$-approximate CCE --- or more concisely, $\varepsilon$-CCE --- if $\NEgap(\pi) \leq \varepsilon$.

\paragraph{Generative model\,/\,simulator.} In reality, we oftentimes do not have access to perfect descriptions (e.g., accurate knowledge of the transition kernel $P$) of the Markov game under consideration; instead, one has to learn the true model on the basis of data samples. 
When it comes to the data generating mechanism, this paper assumes access to a generative model (also called a simulator) \citep{kearns2002sparse,kakade2003sample}:   
in each call to the generative model, 
the learner can choose an arbitrary   $(s,\bm{a},h)\in \cS\times \cA \times [H]$ and obtain an independent sample generated based on the true transition kernel:
\[
	s' \sim P_h(\cdot \mymid s, \bm{a}). 
\]
In words,  a generative model facilitates query of arbitrary state-action-step combinations, which helps alleviate the sampling constraints arising in online episodic settings for exploration.
The goal of the current paper is to compute an $\varepsilon$-approximate equilibrium (either NE or CCE) of $\mathcal{MG}$ with as few samples as possible, 
i.e., using a minimal number of calls to the generative model.

%
%

\section{Sample-efficient learning with a generative model}
\label{sec:alg-theory}

In this section, we put forward an efficient algorithm aimed at learning an $\varepsilon$-approximate equilibrium with the assistance of a generative model, 
and demonstrate its sample optimality for the full $\varepsilon$-range.


\subsection{Algorithm description}
\label{sec:algorithm}

We now describe the proposed algorithm, 
which is inspired by the optimism principle and the FTRL algorithm for online/adversarial learning. 
Following the dynamic programming approach \citep{bertsekas2017dynamic}, our algorithm employs backward recursion from step $h=H$ back to $h=1$; in fact, we shall finish the sampling and learning processes for step $h$ before moving backward to step $h-1$.   
For each $h$, the $i$-th player calls the generative model for $K$ rounds, with each round drawing $SA_i$ independent samples;
as a result, the total sample size is given by $KSH\sum_{i = 1}^m A_i$. 
In what follows, let us first introduce some convenient notation that facilitates our exposition of the algorithm.

\paragraph{Notation.}  Consider any step $h\in [H]$, any player $i\in [m]$, and any data collection round $k\in [K]$. The algorithm maintains the following iterates, whose notation is gathered here with their formal definitions introduced later. 
\begin{itemize}
	\item $\widehat{V}_{i, h}\in \mathbb{R}^S$ represents the final estimate of the value function at step $h$ by the $i$-th player; in particular, we set   $\widehat{V}_{i,H+1}=0$. 
	\item $Q_{i, h}^k\in \mathbb{R}^{SA_i}$ represents the Q-function estimate of the $i$-th player at step $h$ after the $k$-th round of data collection. 

	\item $q_{i, h}^k\in \mathbb{R}^{SA_i}$ stands for a certain ``one-step-look-ahead'' Q-function estimate of the $i$-th player at step $h$ using samples collected in the $k$-th round.

	\item $r_{i, h}^k \in \mathbb{R}^{SA_i}$ denotes the sample reward vector for step $h$ received by the $i$-th player in the $k$-th round. 

	\item $P_{i, h}^k \in \mathbb{R}^{SA_i\times S}$ 
		denotes the empirical probability transition matrix for step $h$ constructed using the samples collected by the $i$-th player in the $k$-th round. 


	\item $\beta_{i, h} \in \mathbb{R}^{S}$ denotes the bonus vector chosen by the $i$-th player at step $h$ during final value estimation. 
	

	\item $\pi_{i, h}^k: \mathcal{S}\rightarrow \Delta(\mathcal{A}_i)$ denotes the policy iterate of the $i$-th player at step $h$ before the beginning of the $k$-th round of data collection; in particular, we set $\pi_{i,h}^1$ to be uniform, namely, $\pi_{i, h}^1(a_i \mymid s)= 1/A_i$ for any $(s,a_i)\in \cS\times \cA_i$. 
	
		

\end{itemize}
Crucially, the above objects are all constructed from the perspective of a single player, 
and hence resemble those needed to operate a ``single-agent'' MDP (as opposed to MARL).   
As such, the complexity of storing/updating the above objects only scales with the aggregate size of the individual action space, 
rather than the size of the product action space.

\paragraph{Main steps of the proposed algorithm.}
As mentioned above, our algorithm collects multiple rounds of independent samples for each $h$. 
In what follows, let us describe the proposed procedure for the $i$-th player in the $k$-th round for step $h$.
\begin{itemize}
	\item[1.] {\em Sampling and model estimation.} 
		For each $(s,a_i)\in \cS\times \cA_i$, 
		draw an {\em independent} sample as follows
		\begin{subequations}
		\label{eq:sampling-model-estimation}
		\begin{equation}
			s'_{k,h,s,a_i} \sim P_h\big(\cdot \mymid s,\bm{a}(k,h,s,a_i)\big)
			\qquad \text{and} \qquad
			r_{k,i, h,s,a_i}=r_{i, h}\big(s, \bm{a}(k,h,s,a_i) \big), 
		\end{equation}
		where $\bm{a}(k,h,s,a_i)=[a_j(k,h,s,a_i)]_{1\leq j\leq m}\in \cA$ consists of independent individual actions drawn from 
		\begin{equation}
			a_j(k,h,s,a_i) \overset{\text{ind.}}{\sim} \pi_{j,h}^k(\cdot\mymid s) ~~~~(j\neq i)
			\qquad \text{and} \qquad
			a_i(k,h,s,a_i) = a_i. 
		\end{equation}
		These samples are then employed to construct the sample reward vector $r_{i, h}^k \in \mathbb{R}^{SA_i}$ and empirical probability transition kernel $P_{i, h}^k \in \mathbb{R}^{SA_i\times S}$ such that
		\begin{equation}
			r_{i, h}^{k}(s,a_i)=r_{k,i,h,s,a_i}
			\qquad\text{and}\qquad
			P_{i, h}^{k}(s'\mymid s,a_i)=
			\begin{cases}
			1, & \text{if }s'=s'_{k,h,s,a_i} \\
			0, & \text{else}
			\end{cases}
		\label{eq:construction-model-k}
		\end{equation}
		\end{subequations}
		for all $(s,a_i,s')\in \cS\times \cA_i\times \cS$. Note that the $i$-th player only needs to compute \eqref{eq:construction-model-k}, without the need of  directly observing the other players' actions.

	\item[2.] {\em Q-function estimation.} Following the dynamic programming approach, we first compute the ``one-step-look-ahead'' Q-function estimate as follows
\begin{align}
	q_{i, h}^k &= r_{i, h}^k + P_{i, h}^k\widehat{V}_{i, h+1}. 
	\label{eq:overline-q-update}
\end{align}
We then adopt the update rule of Q-learning:
\begin{align}
	Q_{i, h}^k &= (1-\alpha_k)Q_{i, h}^{k-1} + \alpha_k q_{i, h}^k,
	\label{eq:overline-Q-update-12}
\end{align}
where $0<\alpha_k<1$ is the learning rate. Applying \eqref{eq:overline-Q-update-12} recursively and using the quantities defined in \eqref{def:alpha-i-k}, we easily arrive at the following expansion: 
\begin{align}
	Q_{i, h}^k = \sum_{j = 1}^k \alpha_j^k q_{i, h}^j.
	\label{eq:overline-Q-expansion-12}
\end{align}

	\item[3.] {\em Policy updates.} Once the Q-estimates are updated, we adopt the exponential weights strategy to update the policy iterate of the $i$-th player as follows
		\begin{align}
			\pi_{i, h}^{k+1}(a_i\mymid s) &= \frac{\exp\big(\eta_{k+1} Q_{i, h}^k(s,a_i)\big)}{\sum_{a'\in \cA_i} \exp\big(\eta_{k+1} Q_{i, h}^k(s,a')\big)},
			\qquad \forall (s,a_i)\in \cS\times \cA_i,
			\label{eq:policy-update-exponential-12}
		\end{align}
		where $\eta_{k+1}>0$ is another learning rate associated with policy updates (to be specified shortly). 
		In fact, this subroutine implements the Follow-the-Regularized-Leader strategy \citep{shalev2012online}: 
		\begin{align}
			\pi_{i, h}^{k+1}(\cdot \mymid s) 
			&= \arg\min_{\mu \in \Delta(\cA_i)} ~
			\bigg\{ - \big\langle \mu, Q_{i, h}^k(s,\cdot) \big\rangle + \frac{1}{\eta_{k+1}} F(\mu) \bigg\}, 
			\label{eq:policy-update-exponential-FTRL}
		\end{align}
		where the regularizer $F(\cdot)$ is chosen to be the negative entropy function $F(\mu) \coloneqq \sum_{a\in \cA_i}\mu(a)\log \big(\mu(a)\big)$.

\end{itemize}
After carrying out $K$ rounds of the above procedure, our final policy estimate $\widehat{\pi}:  \cS \times [H] \rightarrow \Delta(\cA)$ and the value estimate $\widehat{V}_{i, h}: \cS\rightarrow \mathbb{R}$ for step $h$ 
are taken respectively to be 
\begin{subequations}
\begin{align}
	\widehat{V}_{i, h}(s) &= \min\left\{\sum_{k = 1}^K \alpha_k^K \Big\langle \pi_{i, h}^k(\cdot \mymid s), \,q_{i, h}^k(s,\cdot) \Big\rangle + \beta_{i, h}(s), ~ H-h+1 \right\} 
	\qquad \text{and} 
	\label{eq:V-max-output} \\
	\widehat{\pi}_h( \bm{a} \mymid s) &= \sum_{k=1}^K \alpha_k^K \prod_{i=1}^m \pi_{i, h}^k(a_i \mymid s)  \label{eq:policy-output-h}
\end{align}
\end{subequations}
for any $\big(s,\bm{a}=[a_1,\ldots,a_m]\big)\in \cS\times \cA$, 
where $\{\alpha_k^K\}$ is defined in \eqref{def:alpha-i-k}
and $\beta_{i, h}(s) \geq 0$ is some bonus term (taking the form of some data-driven upper confidence bound) 
to be specified momentarily.  
It is worth pointing out that the final policy \eqref{eq:policy-output-h} takes the form of {\em a mixture of product policies}. 
In the special case of two-player zero-sum MGs, 
we can alternatively output a product policy 
\begin{equation}
	(\text{two-player zero-sum MGs}) \qquad \widehat{\pi} = \widehat{\pi}_1 \times \widehat{\pi}_2 ,
\end{equation}
where for each $i=1,2$, we take $\widehat{\pi}_{i}=\{\widehat{\pi}_{i,h}\}_{1\leq h\leq H}$   with 
		 $\widehat{\pi}_{i,h} =  \sum_{k=1}^K \alpha_{k}^{K}  \pi_{i,h}^{k} 	$.

The whole procedure is summarized in Algorithm~\ref{alg:Markov-games-simulator}.

\paragraph{Choices of learning rates.}  
Thus far, we have not yet specified the two sequences of learning rates, which we describe now. 
The learning rates associated with Q-function updates are set to be rescaled linear, namely, 
\begin{align}
	\alpha_k = \frac{ \calpha \log K}{k-1 + \calpha \log K}, \qquad k=1, 2, \ldots
	\label{eq:alphak-choice}
\end{align}
for some constant $\calpha\geq 24$. 
In addition, the learning rates associated with policy updates are chosen to be:
\begin{align}
	\eta_{k+1} &= \sqrt{\frac{\log K}{\alpha_kH}}, \qquad k=1, 2, \ldots
	\label{eq:eta-k-choice}
\end{align}
%

\paragraph{Choices of bonus terms.} 
It remains to specify the bonus terms, which are selected based on fairly intricate upper confidence bounds. 
This constitutes a key --- and perhaps the most challenging --- component of our algorithm design.  
Specifically, we take
%
%
\begin{align}
%
	\beta_{i,h} (s) &= 
	\cb\sqrt{\frac{\log^3 \big( \frac{KS\sum_i A_i}{\delta} \big) }{KH}}\sum_{k = 1}^K 
	\alpha_k^K \Bigg\{ \mathsf{Var}_{\pi_{i,h}^{k}(\cdot\mid s)}\Big(q_{i,h}^k(s,\cdot)\Big) + H\Bigg\}
	\label{eq:choice-bonus-terms-V}
\end{align}
%
%
for any $(i,s,h)\in [m] \times \cS\times [H]$, where $\cb > 0$ is some sufficiently large constant; 
see also \eqref{eq:notation-Epi-Varpi} for the definition of the variance-style quantity. 
As in previous works, the bonus terms, which are chosen carefully in a data-driven fashion, need to compensate for the uncertainty incurred during the estimation process. 
%
%
%


\begin{algorithm}[t]
	\textbf{Input:} number of rounds $K$ for each step, learning rates $\{\alpha_k\}$ (cf.~\eqref{eq:alphak-choice}) and $\{\eta_{k+1}\}$ (cf.~\eqref{eq:eta-k-choice}). \\
	{\color{blue}\tcp{set initial value estimates to 0, and initial policies to uniform distributions.}}
	\textbf{Initialize:}  for any $i\in [m]$ and any $(s,a_i,h)\in \cS\times \cA_i \times[H]$, set
	$\widehat{V}_{i,H+1}(s)=Q_{i, h}^0(s,a_i)=0$ and $\pi_{i,h}^1(a_i\mymid s)=1/A_i$.  \\
	\For{$ h = H$ \KwTo $1$}{
	\For{$ k = 1$ \KwTo $K$}{
	\For{$ i = 1$ \KwTo $m$}{
		{\color{blue}\tcp{draw independent samples, and construct empirical models.}}
		$\big(r_{i,h}^k,P_{i,h}^k\big)$ $\leftarrow$ 
		\tt{sampling}$\big(i, h, \pi_h^k = \{\pi_{j, h}^k\}_{j\in[m]}\big)$. {\color{blue}\tcc{see Algorithm~\ref{alg:sampling-function}.}}
		
		{\color{blue}\tcp{update Q-estimates with upper confidence bounds.}} 
		Compute 
$q_{i, h}^k = r_{i, h}^k + P_{i, h}^k\widehat{V}_{i, h+1}$, 
and update
\vspace{-0.5em}
\begin{align*}
	Q_{i, h}^k = (1-\alpha_k)Q_{i, h}^{k-1} + \alpha_k q_{i, h}^k.
\end{align*}
{\color{blue}\tcp{update policy estimates using FTRL.}}
	\For{$(s,a_i)\in \cS\times \cA_i$}{
		\vspace{-2em}
		\begin{align*}
			\pi_{i, h}^{k+1}(a_i\mymid s) = \frac{\exp\big(\eta_{k+1} Q_{i, h}^k(s,a_i)\big)}{\sum_{a'} \exp\big(\eta_{k+1} Q_{i, h}^k(s,a')\big)}.
		\end{align*}
	}
	}
	}
	{\color{blue}\tcp{output the final value estimate for step $h$.}}
	\For{$ i = 1$ \KwTo $m$}{
	\label{line:policy-V-output} 
	\vspace{-2em}
	\begin{align}
		\widehat{V}_{i,h}(s)&=\min\left\{ \sum_{k=1}^{K}\alpha_{k}^{K}\big\langle\pi_{i,h}^{k}(\cdot\mymid s),\,q_{i,h}^{k}(s,\cdot)\big\rangle+\beta_{i,h}(s),~H-h+1\right\} , 
		~\forall s\in  \cS, \notag
	\end{align}
	where $\beta_{i,h}$ is given in \eqref{eq:choice-bonus-terms-V}. \label{eq:line-number-policy-update}
	}
	}
	%
	
%
	
	\If{ $\mathcal{MG}$ is a two-player zero-sum Markov game }
	{
		\textbf{output:} $\widehat{\pi}_{1}\times \widehat{\pi}_{2}$, where for any $i=1,2$, $\widehat{\pi}_{i}=\{\widehat{\pi}_{i,h}\}_{1\leq h\leq H}$   with 
		 $\widehat{\pi}_{i,h} =  \sum_{k=1}^K \alpha_{k}^{K}  \pi_{i,h}^{k} 	$. 
		 \label{line:output-two-player-zero-sum}
	}
	\If{ $\mathcal{MG}$ is a multi-player general-sum Markov game } {
	\textbf{output:} $\widehat{\pi}=\{\widehat{\pi}_h\}_{1\leq h\leq H}$, 
	where $\widehat{\pi}_{h} =  \sum_{k=1}^K \alpha_{k}^{K} \big(  \pi_{1,h}^{k} \times \cdots \times \pi_{m,h}^{k} \big)	$. 
}

	\caption{\myalg.\label{alg:Markov-games-simulator}}

\end{algorithm}

\begin{algorithm}[t]

	\textbf{Initialize:} $\overline{r}=0 \in \mathbb{R}^{SA_i}$, and $\overline{P}=0 \in \mathbb{R}^{SA_i\times S}$. 
	
	\For{$(s,a_i)\in \cS\times \cA_i$}{
		Draw an independent sample from the generative model:
		\begin{equation}
			s'_{s,a_i} \sim P_h\big(\cdot \mymid s, \bm{a}(s,a_i)\big) ,
		\end{equation}
		where $\bm{a}(s,a_i)= [a_{j}(s,a_i) ]_{1\leq j\leq m}$ is composed of independent individual actions drawn from%
		\begin{equation}
			a_{j}(s,a_i) \overset{\text{ind.}}{\sim} \pi_{j, h}(\cdot \mymid s) ~~~~ (j\neq i)
			\qquad \text{and} \qquad
			a_{i}(s,a_i) = a_i.
		\end{equation}

		Set $\overline{r}(s,a_i)=r_{i,h}\big( s,\bm{a}(s,a_i) \big)$ and $\overline{P}\big(s'_{s,a_i}\mymid s,a_i \big)=1$. 
	}

	\textbf{Return:} $\big( \overline{r}, \overline{P} \big)$.

	\caption{Auxiliary function \tt{sampling}$\big(i, h, \pi_h = \{\pi_{j, h}\}_{j\in[m]}\big)$.\label{alg:sampling-function}}
\end{algorithm}


\subsection{Main results}
\label{sec:main-results}

As it turns out, the proposed algorithm is tractable and provably sample-efficient. 
We begin by characterizing its sample complexity when learning Nash equilibria in two-player zero-sum MGs, 
and then shift attention to learning CCE in multi-player general-sum MGs (given the intractability of learning NEs in general).
\begin{theorem}[NE for two-player zero-sum MGs]
\label{eq:main-result-zero-sum-special}
Consider a two-player zero-sum Markov game, 
and consider any $\varepsilon \in (0,H]$ and any $0<\delta < 1$. Suppose that 
\begin{equation}
	K\geq \frac{c_{\mathsf{k}}H^3 \log^4 \big( \frac{KS(A_1+A_2)}{\delta}\big) }{\varepsilon^2} 
	\label{eq:K-sample-size-thm-special}
\end{equation}
for some large enough universal constant $c_{\mathsf{k}}>0$.  
With probability at least $1-\delta$, the product policy $\widehat{\pi}_1\times \widehat{\pi}_2$ computed by Algorithm~\ref{alg:Markov-games-simulator} is an $\varepsilon$-approximate Nash equilibrium, i.e., 
its sub-optimality gap (cf.~\eqref{eq:defn-NE-gap})  obeys
\begin{align*}
	\NEgap\big(\widehat{\pi}_1\times \widehat{\pi}_2\big) \leq \varepsilon .
\end{align*}
\end{theorem}
%
%
\begin{theorem}[CCE for multi-player general-sum MGs]
\label{eq:main-result-zero-sum}
Consider an $m$-player general-sum Markov game, and consider any $\varepsilon \in (0,H]$ and any $0<\delta < 1$. Suppose that 
\begin{equation}
	K\geq \frac{c_{\mathsf{k}}H^3 \log^4 \big( \frac{KS\sum_{i=1}^m A_i}{\delta} \big) }{\varepsilon^2}
	\label{eq:K-sample-size-thm}
\end{equation}
for some large enough universal constant $c_{\mathsf{k}}>0$. 
With probability at least $1-\delta$, the joint policy $\widehat{\pi}$ returned by Algorithm~\ref{alg:Markov-games-simulator} is an $\varepsilon$-approximate CCE, i.e., 
its sub-optimality gap (cf.~\eqref{eq:defn-NE-gap}) obeys
\begin{align*}
	\NEgap\big(\widehat{\pi}\big) \leq \varepsilon .
\end{align*}
\end{theorem}

Theorems~\ref{eq:main-result-zero-sum-special}-\ref{eq:main-result-zero-sum} establish sample complexity upper bounds for the proposed algorithm, which we take a moment to interpret as follows. 
The proofs of these two theorems are postponed to Section~\ref{sec:analysis}.

\paragraph{Sample complexity.}
When a generative model is available, Theorems~\ref{eq:main-result-zero-sum-special}-\ref{eq:main-result-zero-sum} assert that the total number of samples (i.e., $KSH\sum_i A_i$) needed for Algorithm~\ref{alg:Markov-games-simulator} 
to work is 
\begin{align}
	\begin{cases}
		\widetilde{O}\big( \frac{H^4S(A_1+A_2)}{\varepsilon^2} \big), \qquad & \text{for learning an }\varepsilon \text{-NE in two-player zero-sum MGs};\\
		\widetilde{O}\big( \frac{H^4S\sum_{i=1}^m A_i}{\varepsilon^2} \big), & \text{for learning an }\varepsilon \text{-CCE in multi-player general-sum MGs}.\\
	\end{cases}
	\label{eq:sample-complexity-Thm-1}
\end{align}
As far as we know, our theorems deliver the first results that uncover the plausibility of simultaneously overcoming the long-horizon barrier and the curse of multi-agents. 
Let us compare \eqref{eq:sample-complexity-Thm-1} with prior art. 
\begin{itemize}
	\item {\em NE in two-player zero-sum MGs. }
First, consider learning $\varepsilon$-NE policies in two-player zero-sum MGs.  
In comparison to \citet{zhang2020marl} (cf.~\eqref{eq:Zhang-plug-in-result}), 
our result reveals that what ultimately matters is the total number of individual actions (i.e., $A_1+A_2$) as opposed to the total number $A_1A_2$ of possible joint actions;   
additionally, our results exhibit improved horizon dependency (by a factor of $H^2$) compared to \citet{bai2020near,jin2021v} (see \eqref{eq:Jin-v-learning-result}), 
although we remark that the online sampling protocol therein is clearly more restrictive than a generative model. 
 
	\item {\em CCE in multi-player general-sum MGs (for a fixed $m$).}  
		Similar messages carry over to the task of learning multi-player general-sum MGs when the number of players $m$ is a fixed constant. 
		\citet{liu2021sharp} provided the first non-asymptotic result on  learning CCE in the exploration setting; 
		the model-based algorithm studied therein learns an $\varepsilon$-CCE using
		\begin{equation}
			\widetilde{O}\bigg( \frac{H^5S^2 \prod_{i=1}^m A_i}{\varepsilon^2} \bigg) ~~\text{samples}
			\qquad \text{or} \qquad
			\widetilde{O}\bigg( \frac{H^4S^2 \prod_{i=1}^m A_i}{\varepsilon^2} \bigg) ~~\text{episodes}
		\end{equation}
		which is sub-optimal in terms of the dependency on both $H$ and $S$ and suffers from the curse of multiple agents. 
		A more recent strand of works focused on a type of online RL algorithms called V-learning, which exploited the effectiveness of adversarial learning subroutines in overcoming the curse of multi-agents \citep{mao2022provably,song2021can,jin2021v}; along this line,  the state-of-the-art sample complexity bound is \citep{jin2021v}: 
		\begin{equation}
			\widetilde{O}\bigg( \frac{H^6S \max_{1\leq i \leq m} A_i}{\varepsilon^2} \bigg) ~~\text{samples}
			\qquad \text{or} \qquad
			\widetilde{O}\bigg( \frac{H^5S \max_{1\leq i \leq m} A_i}{\varepsilon^2} \bigg) ~~\text{episodes},
		\end{equation}
		which remains suboptimal in terms of the horizon dependency. 
		As a drawback of these works, the policy returned by V-learning is non-Markovian, an issue that has been recently addressed by \citet{daskalakis2022complexity} at the price of a much higher sample complexity.  
		It is worth emphasizing that all these works assume the online exploration setting as opposed to the scenario with a generative model. 
		
\end{itemize}

\paragraph{Minimax optimality.} To assess the tightness of our result \eqref{eq:sample-complexity-Thm-1}, it is helpful to look at the information-theoretic limit. 
Following the minimax lower bound for single-agent MDPs \citep{azar2013minimax,li2022settling}, 
one can develop a minimax sample complexity lower bound for Markov games (w.r.t.~finding either an $\varepsilon$-NE or an $\varepsilon$-CCE) that scales as 
\begin{align}
	\text{(minimax lower bound)} \qquad  \frac{H^4S  \max_{1\leq i\leq m} A_i }{\varepsilon^2} 	
	\label{eq:minimax-lower-sample-complexity}
\end{align}
modulo some logarithmic factor;
see Appendix~\ref{sec:lower} for a formal statement and its proof.  
Taking this together with \eqref{eq:sample-complexity-Thm-1} confirms the minimax optimality of our algorithm (up to logarithmic terms) when the number $m$ of players is fixed or grows only logarithmically in problem parameters.

\paragraph{No burn-in sample size and full $\varepsilon$-range. } 
It is noteworthy that the validity of our sample complexity bound \eqref{eq:sample-complexity-Thm-1} is guaranteed for the entire range of $\varepsilon$-levels (i.e., any $\varepsilon \in (0,H]$). This feature is particularly appealing in the data-starved applications, 
as it implies that there is no burn-in sample size needed for our algorithm to work optimally.  

\paragraph{Miscellaneous properties of our algorithm.} 
Finally, we would like to remark in passing that our learning algorithm enjoys several properties that might be practically appealing.  
For instance, the output policies are Markovian in nature, which depend only on the current state $s$ and step number $h$.  
This is enabled thanks to the availability of the generative model, which allows us to settle the sampling and learning process for step $h+1$ completely before moving backward to step $h$; 
in contrast, the online sampling protocol studied in \citet{bai2020near,jin2021v} cannot be implemented in this way without incurring information loss. 
In addition, our algorithm can be carried out in a decentralized fashion (except that the final estimate $\widehat{\pi}$ needs to aggregate policy iterates from all players), with each player acting in a symmetric yet independent manner (without the need of knowing each other's individual action). Our algorithm is also ``rational'' in the sense that it converges to the best-response policy of a player if all other players freeze their policies. 
All this is achieved under minimal sample complexity with the aid of the generative model.


\section{Regret bounds for FTRL via variance-type quantities}
\label{sec:FTRL}

Before embarking on our analysis for Markov games, we take a detour to study the celebrated Follow-the-Regularized-Leader algorithm for online weighted linear optimization, 
which plays a central role in the analysis of Markov games.

\subsection{Setting: online learning for weighted linear optimization}
\label{sec:online-learning-weighted}

Let $\ell_1,\ldots,\ell_n \in \mathbb{R}^A$ represent an arbitrary sequence of {\em non-negative} loss vectors. 
We focus on the following setting of online learning or adversarial learning \citep{lattimore2020bandit}: 
in each round $k$, 
\begin{itemize}
	\item[1.] the learner makes a randomized prediction by choosing a distribution $\pi_k\in \Delta(\cA)$ over the actions in $\cA=\{1,\cdots,A\}$;
	\item[2.] subsequently, the learner observes the loss vector $\ell_k$, which is permitted to be adversarially chosen.
\end{itemize}

To evaluate the performance of the learner, we resort to a regret metric w.r.t.~a certain weighted linear objective function. 
To be precise, 
consider a non-negative sequence $\{\alpha_k \}_{1\leq k\leq n}$ with $0\leq \alpha_k\leq 1$;  
for each $1\leq k\leq n$, we define recursively the following weighted average of the loss vectors:    
\begin{align*}
	L_0 = 0 \qquad \text{and} \qquad L_k = (1-\alpha_k) L_{k-1} + \alpha_k  \ell_k, \quad  k\geq 1, 
\end{align*}
which can be easily shown to enjoy the following expression
\begin{align*}
	L_k 
	= \sum_{i = 1}^k \alpha_i^k \ell_k 
\end{align*}
with $\alpha_i^k$ defined in \eqref{def:alpha-i-k}. When the sequential predictions made by the learner are $\{\pi_k\}_{k\geq 1}$, 
we define the associated regret w.r.t.~the above weighted sum of loss vectors as follows:  
\begin{align}
	R_n &\defn \max_{a\in \cA} R_n(a) 
\qquad \text{with }\,
	R_n(a) \defn \sum_{k = 1}^n \alpha_k^n \langle \pi_k, \ell_k\rangle - \sum_{k = 1}^n \alpha_k^n  \ell_k(a),  
	\label{eq:regret-defn-FTRL}
\end{align}
which compares the learner's performance (i.e., the expected loss of the learner over time if it draws actions based on $\pi_k$ in round $k$) against that of the best {\em fixed} action in hindsight.


\subsection{Refined regret bounds for FTRL}
\label{sec:refined-regret-FTRL}

\paragraph{Follow-the-Regularized-Leader.}
The FTRL algorithm \citep{shalev2007primal,shalev2007online} tailored to the above online optimization setting adopts the following update rule:
\begin{align}
	\pi_{k+1} &= \arg\min_{\pi \in \Delta(\cA)} ~ \bigg\{ \langle \pi, L_k\rangle + \frac{1}{\eta_{k+1}}F(\pi) \bigg\},
	\qquad k= 1,2,\ldots
	\label{eq:FTRL-update-general}
\end{align}
%
where $\eta_{k+1} >0 $ denotes the learning rate, 
and $F(\cdot)$ is some convex regularization function employed to stabilize the learning process \citep{shalev2012online}.  
Throughout this section, we restrict our attention to negative-entropy regularization, namely,
\begin{align*}
	F(\pi) = \sum_{a\in \cA} \pi(a)\log\big(\pi(a)\big), 
\end{align*}
which allows one to express the FTRL update rule as the following exponential weights strategy (see, e.g., \citet[Section 28.1]{lattimore2020bandit})
\begin{align}
	\pi_{k+1}(a) =  \frac{ \exp\big( -\eta_{k+1}L_{k} (a) \big) }{  \sum_{a'\in \cA} \exp\big(-\eta_{k+1}L_{k}(a') \big) }
	\qquad \text{for all } a\in \cA.
	\label{eq:FTRL-update-entropy}
\end{align}
This update rule is also intimately connected to online mirror descent \citep{lattimore2020bandit}.

\paragraph{Refined regret bounds via variance-style quantities.}

As it turns out, the regret of FTRL can be upper bounded by certain (weighted) variance-type quantities, 
as asserted by the following theorem.


\begin{theorem} 
\label{thm:FTRL-refined}
Suppose that $0< \alpha_1 \leq 1$ and $\eta_1 = \eta_2(1-\alpha_1)$. 
Also, assume that $0< \alpha_k < 1$ and $0<\eta_{k+1}(1-\alpha_k) \leq \eta_{k}$ for all $k\geq 2$. 
In addition, define
\begin{align}
\widehat{\eta}_k \defn
\begin{cases}
\eta_2, & \text{if }  k = 1, \\
\frac{\eta_{k}}{1-\alpha_k},\quad & \text{if }  k > 1.
\end{cases}.
\end{align}
Then the regret (cf.~\eqref{eq:regret-defn-FTRL}) of the FTRL algorithm satisfies
%
\begin{align}
R_n &\le \frac{5}{3}\sum_{k=1}^{n}\alpha_{k}^{n}\widehat{\eta}_{k}\alpha_{k}\mathsf{Var}_{\pi_{k}}(\ell_{k})
	+ \frac{\log A}{\eta_{n+1}}+3\sum_{k=1}^{n}\alpha_{k}^{n}\widehat{\eta}_{k}^{2}\alpha_{k}^{2}\big\|\ell_{k}\big\|_{\infty}^{3}\ind\bigg(\widehat{\eta}_{k}\alpha_{k}\big\|\ell_{k}\big\|_{\infty}>\frac{1}{3}\bigg), 
	\label{eq:FTRL-refined}
\end{align}
where for any $\ell\in \mathbb{R}^A$ and any $\pi \in \Delta(\cA)$ we define
\begin{align*}
	\mathsf{Var}_{\pi}(\ell) \defn \sum_a \pi(a)\Big( \ell(a) - \sum_{a^{\prime}} \pi(a^{\prime}) \ell(a^{\prime})\Big)^2.
\end{align*}
\end{theorem}

\begin{remark} Note that the FTRL algorithm and the data generating process in this section are both described in a completely deterministic manner; no randomness is involved in the above theorem even though we introduce the variance-style quantities. 
\end{remark}
The proof of Theorem~\ref{thm:FTRL-refined} is postponed to Appendix~\ref{sec:proof-thm:FTRL-refined}. 
Let us take a moment to discuss the key distinction between Theorem~\ref{thm:FTRL-refined} and prior theory. 
\begin{itemize}
	\item
A key term in the regret bound \eqref{eq:FTRL-refined} is a weighted sum of the ``variance-style'' quantities $\{\mathsf{Var}_{\pi_{k}}(\ell_{k})\}$. 
In comparison, prior regret bounds typically involve the norm-type quantities (e.g., the infinity norms  $\{\|\ell_{k}\|_{\infty}^2\}$) as opposed to the ``variances''; 
see, for instance,  \citet[Corollary 28.8]{lattimore2020bandit} for a representative existing regret bound that takes the form of the sum of $\{\|\ell_{k}\|_{\infty}^2\}$ that takes the form of the sum of $\{\|\ell_{k}\|_{\infty}^2\}$.\footnote{Note that the Bregman divergence generated by the negative entropy function is the (generalized) KL divergence \citep{beck2017first}, which is strongly convex w.r.t.~$\|\cdot\|_1$ due to Pinsker's inequality. Additionally, the dual norm of $\|\cdot\|_1$ is the infinity norm.} 
While $\mathsf{Var}(\ell_k)\leq \| \ell_k \|_{\infty}^2$ is orderwise tight in the worst-case scenario for a given iteration $k$, 
exploiting the problem-specific variance-type structure across time is crucial in sharpening the horizon dependence in many RL problems (e.g., \citet{azar2013minimax,jin2018q,li2022settling,li2021breaking}). 

	\item The careful reader would  remark that the final term of \eqref{eq:FTRL-refined} relies on the infinity norm $\|\ell_k\|_{\infty}$ as well. 
		Fortunately, when the products of the learning rates $\widehat{\eta}_k \alpha_k$ are chosen to be diminishing (which is the case in our analysis for Markov games), 
		the number of iterations obeying $\widehat{\eta}_{k}\alpha_{k}\|\ell_{k}\|_{\infty}>1/3$ 
		is reasonably small,  thus ensuring that this term does not exert too much of an influence on the regret bound. 

\end{itemize}

\section{Proof of Theorems~\ref{eq:main-result-zero-sum-special}-\ref{eq:main-result-zero-sum}}
\label{sec:analysis}


To begin with, we claim that Theorem~\ref{eq:main-result-zero-sum-special} is a direct consequence of Theorem~\ref{eq:main-result-zero-sum}. 
Towards this, note that in a two-player zero-sum Markov game, it is self-evident that $\widehat{\pi}_{-1} = \widehat{\pi}_{2}$ and $\widehat{\pi}_{-2} = \widehat{\pi}_{1}$ (see line \ref{line:output-two-player-zero-sum} of Algorithm~\ref{alg:Markov-games-simulator}).
Consequently, Theorem~\ref{eq:main-result-zero-sum} (if it is valid) reveals that
\begin{align}
	\varepsilon \ge \NEgap(\widehat{\pi}; s) &= \max\Big\{ V_{1, 1}^{\star,\widehat{\pi}_{-1}}(s)-V_{1, 1}^{\widehat{\pi}}(s), \, V_{2, 1}^{\star,\widehat{\pi}_{-2}}(s)-V_{2, 1}^{\widehat{\pi}}(s)\Big\} \notag\\
	&= \max\Big\{ V_{1, 1}^{\star,\widehat{\pi}_{2}}(s)-V_{1, 1}^{\widehat{\pi}}(s), \, V_{2, 1}^{\widehat{\pi}_{1}, \star}(s)-V_{2, 1}^{\widehat{\pi}}(s)\Big\},
	\qquad\quad \text{for all } s\in \cS.
	\label{eq:gap-two-player}
\end{align}
Moreover, recalling that $r_{1, h} = -r_{2, h}$ for all $h\in [H]$, one has $V_{1, 1}^{\pi}(s) = -V_{2, 1}^{\pi}(s)$ for any joint policy profile $\pi$, 
which taken collectively with \eqref{eq:gap-two-player} results in
\begin{align*}
V_{1, 1}^{\star,\widehat{\pi}_{2}}(s)-V_{1, 1}^{\widehat{\pi}_1 \times \widehat{\pi}_2}(s) &= V_{1, 1}^{\star,\widehat{\pi}_{2}}(s)+V_{2, 1}^{\widehat{\pi}_1 \times \widehat{\pi}_2}(s) 
	\le V_{1, 1}^{\star,\widehat{\pi}_{2}}(s)+V_{2, 1}^{\widehat{\pi}_1,\star}(s) \\
&= V_{1, 1}^{\star,\widehat{\pi}_{2}}(s)-V_{1, 1}^{\widehat{\pi}}(s) + V_{2, 1}^{\widehat{\pi}_1,\star}(s)-V_{2, 1}^{\widehat{\pi}}(s) \le 2\varepsilon.
\end{align*}
Analogously, one has $V_{2, 1}^{\widehat{\pi}_{1}, \star}(s)-V_{2, 1}^{\widehat{\pi}_1\times \widehat{\pi}_2 }(s) \le 2\varepsilon$. 
Replacing $\varepsilon$ with $\varepsilon/2$ immediately establishes Theorem~\ref{eq:main-result-zero-sum-special}.

With the above argument in mind, 
the remainder of this section is devoted to proving Theorem~\ref{eq:main-result-zero-sum}.

\subsection{Preliminaries and notation} \label{sec:preliminary-analysis}
Let us start with some preliminary facts and notation. 
Given that $\varepsilon \leq H$,  the assumption \eqref{eq:K-sample-size-thm} requires
\begin{align}
	K\geq c_{\mathsf{k}}H \log^4 \Big(  \frac{KS\sum_i A_i }{\delta}  \Big) 
	\label{eq:K-general-LB-135}
\end{align}
for some large enough constant $c_{\mathsf{k}}>0$, 
which will be a condition assumed throughout the proof. 
We also gather below several basic facts about our choices of learning rates $\{\alpha_i\}$ (cf.~\eqref{eq:alphak-choice}) and the corresponding quantities $\{\alpha_i^k\}$ (cf.~\eqref{def:alpha-i-k}).
\begin{lemma}
	\label{lem:weight}
For any $k\geq 1$, one has
\begin{subequations}
\begin{equation}
\alpha_{1}=1,\qquad\sum_{i=1}^{k}\alpha_{i}^{k}=1,\qquad\max_{1\leq i\leq k}\alpha_{i}^{k}\leq\frac{2\calpha\log K}{k} .
	\label{eq:alpha-properties}
\end{equation}
In addition, if $k \ge c_{\alpha}\log K + 1$ and $c_{\alpha} \ge 24$, then one has
\begin{equation}
	\max_{1\leq i\leq k/2}\alpha_{i}^{k}\leq 1 / K^6  . 
	\label{eq:alpha-properties-2}
\end{equation}
\end{subequations}
\end{lemma}
\begin{proof}
	The result \eqref{eq:alpha-properties} is standard and has been recorded in previous works (e.g., \citet[Appendix B]{jin2018q}). 
	Regarding \eqref{eq:alpha-properties-2}, we note that for any $i\leq k/2$ and $k\geq c_{\alpha}\log K + 1$, 
	\[
\alpha_{i}^{k}\le\prod_{j=i+1}^{k}(1-\alpha_{j})\le\prod_{j=k/2+1}^{k}(1-\alpha_{j})\le(1-\alpha_{k})^{k/2}\le\Big(1-\frac{c_{\alpha}\log K}{2k}\Big)^{k/2}\le\exp\Big(-\frac{c_{\alpha}\log K}{4}\Big)\le\frac{1}{K^{6}},
\]
where we have used the fact that $\alpha_{k}=\frac{\calpha\log K}{k-1+\calpha\log K}\geq\frac{\calpha\log K}{2k}$
and the assumption $\calpha\geq 24$.
\end{proof}
Additionally, recognizing the definition in \eqref{eq:overline-q-update} and the upper bound $\widehat{V}_{i,h+1}(s)\leq H-h$ (cf.~\eqref{eq:V-max-output}), 
we make note of the range of the iterates $\big\{ q_{i,h}^k \big\}$ as follows. 
\begin{lemma}
	\label{lem:range-q}
	For any $i\in [m]$ and any $(h,k,s,a_i)\in [H]\times [K] \times \cS\times \cA_i$, it holds that
	\begin{align}
		0\leq q_{i,h}^k(s,a_i)\leq H-h+1.
	\end{align}
\end{lemma}
%

Next, we introduce several additional notation that helps simplify our presentation of the proof. 
For any policy $\mu: \cS\times [H]\rightarrow \Delta(\cA_i)$, we adopt the convenient notation 
\begin{align*}
	\mu_h(s) \coloneqq \mu_h(\cdot\mymid s) \in \Delta(\cA_i).
\end{align*}
We shall also employ the expectation operator $\mathbb{E}_{h,k-1}[\cdot]$ (resp.~variance operator $\mathsf{Var}_{h,k-1}[\cdot]$) to denote the expectation (resp.~variance) conditional on what happens before the beginning of the $k$-th round of data collection for step $h$ (see Section~\ref{sec:algorithm} about the data collection process).


\subsection{Proof outline}

With the above preliminaries in place, we are in a position to present our analysis. 
Recall that the joint policy $\widehat{\pi}$ 
computed by Algorithm~\ref{alg:Markov-games-simulator} takes the form of a mixture of product policies 
\begin{equation}
	\sum_{k=1}^K\alpha_k^K  \underset{\eqqcolon\, \pi_h^k}{\underbrace{ \big(\pi_{1,h}^k \times  \cdots \times \pi_{m,h}^k  \big) }}  
	\label{eq:defn-pi-h-k-123}
\end{equation}
at step $h$. 
Consequently, the value function under policy  $\widehat{\pi}$ satisfies the following Bellman equation:
\begin{subequations}
\label{eq:defn-V-mu-cross-nu}
\begin{align}
		V_{i,H+1}^{\pihat}(s) &\coloneqq 0 \\
	V_{i,h}^{\pihat}(s) &\coloneqq \sum_{k = 1}^K\sum_{\bm{a} \in \cA} \alpha_{k}^{K}\pi_{h}^k(\bm{a} \mymid s)\Big[r_{i,h}(s,\bm{a}) + \big\langle P_{h}(\cdot\mymid s,\bm{a}), V_{i,h+1}^{\pihat} \big\rangle \Big]
\end{align}
\end{subequations}
for all $ (i,s,h) \in [m]\times \cS\times[H]$. 
%
%
%
To establish Theorem~\ref{eq:main-result-zero-sum}, we seek to prove the following inequality: 
\begin{align}
V_{i, 1}^{\star,\pihat_{-i}}(s)-V_{i, 1}^{\pihat}(s) \le \varepsilon,
	\qquad 1\leq i\leq m,
	\label{eq:target-two-inequalities}
\end{align}
where we remind the reader of the definition of $V_{i, 1}^{\star,\pihat_{-i}}$ in \eqref{eq:defn-optimal-V}.


Towards this, let us introduce the following best-response policy of the $i$-th player: 
\begin{align*}
	\pihatstar_{i} = \big[\pihatstar_{i,h} \big]_{h\in [H]}  \coloneqq \arg\max_{\pi_i': \cS\times[H]\rightarrow \Delta(\cA_i)} V_{i,1}^{\pi_i' \times \pihat_{-i}}. 
\end{align*}
We make note of the following key decomposition: 
\begin{align}
	\label{eq:key_decomposition}
	V_{i,h}^{\star,\pihat_{-i}}-V_{i,h}^{\pihat} \le \big( V_{i,h}^{\star,\pihat_{-i}} - \overline{V}_{i,h}^{\pihatstar_{i} \times \pihat_{-i}} \big) + \big( \overline{V}_{i,h}^{\star,\pihat_{-i}} - \overline{V}_{i,h}^{\pihat} \big) + \big( \overline{V}_{i,h}^{\pihat} - V_{i,h}^{\pihat} \big),
\end{align}
where we define the following auxiliary value functions: 
\begin{subequations}
\begin{align}
	\overline{V}_{i,h}^{\pihatstar_{i}\times \pihat_{-i}}(s) 
	&\coloneqq  \sum_{k=1}^{K}\alpha_{k}^{K} \mathop{\mathbb{E}}\limits _{a_i\sim\pihatstar_{i,h}(s)}\bigg[ r_{i,h}^{k}(s,a_i)+\big\langle P_{i,h}^{k}(\cdot\mymid s,a_i),\overline{V}_{i,h+1}^{\pihatstar_{i}\times \pihat_{-i}}\big\rangle\bigg] 
	,\qquad &&\text{with }\overline{V}_{i,H+1}^{\pihatstar_{i}\times \pihat_{-i}}=0,
 \label{defi:V-mustar}\\
	\overline{V}_{i,h}^{\star,\pihat_{-i}}(s) &\coloneqq \max_{a_i\in \cA_i}\sum_{k = 1}^K \alpha_{k}^{K} \Big[r_{i,h}^k(s, a_i) + \big\langle P_{i,h}^k(\cdot\mymid s, a_i), \overline{V}_{i,h+1}^{\star,\pihat_{-i}} \big\rangle \Big], \qquad
	&&\text{with }\overline{V}_{i,H+1}^{\star,\pihat_{-i}} = 0, \label{defi:V-star}\\
	\overline{V}_{i,h}^{\pihat}(s) &\coloneqq \sum_{k = 1}^K \alpha_{k}^{K} \mathop{\mathbb{E}}_{a_i\sim \pi_{i,h}^k(s)}\Big[r_{i,h}^k(s, a_i) + \big\langle P_{i,h}^k(\cdot\mymid s, a_i), \overline{V}_{i,h+1}^{\pihat} \big\rangle \Big], \qquad
	&&\text{with }\overline{V}_{i,H+1}^{\pihat} = 0. \label{defi:V-muhat}
\end{align}
\end{subequations}
Here, we have used the elementary fact $\overline{V}_{i,h}^{\pihatstar_{i}\times \pihat_{-i}} \leq \overline{V}_{i,h}^{\star,\pihat_{-i}}$.  
We shall establish bounds for the above terms in \eqref{eq:key_decomposition}, which consists of three steps as outlined below.

%
%
%
%

\paragraph{Step 1: showing that $\widehat{V}_{i,h}$ is an entrywise upper bound on $\overline{V}_{i,h}^{\star,\pihat_{-i}}$.}

The following lemma ascertains that the value  estimate $\widehat{V}_{i,h}$ of the $i$-th player returned by Algorithm~\ref{alg:Markov-games-simulator} is an optimistic estimate of the auxiliary value $\overline{V}_{i,h}^{\star, \pihat_{-i}}$ defined in \eqref{defi:V-star}.
Evidently, this result cannot happen unless the bonus terms are suitably chosen. 
\begin{lemma} \label{lem:UCB}
With probability at least $1-\delta$, it holds that
\begin{align}
	\widehat{V}_{i,h} \ge \overline{V}_{i,h}^{\star, \pihat_{-i}}, 
	\qquad \quad \text{for all } (i , h)\in [m] \times [ H].
	\label{eq:lem-UCB}
\end{align}
\end{lemma}

The proof of this lemma is postponed to Appendix~\ref{sec:proof-lemma:UCB}. Armed with Lemma~\ref{lem:UCB}, 
we can further bound \eqref{eq:key_decomposition} as follows
\begin{align}
		V_{i,h}^{\star,\pihat_{-i}}-V_{i,h}^{\pihat} \le \big( V_{i,h}^{\star,\pihat_{-i}} - \overline{V}_{i,h}^{\pihatstar_{i}\times \pihat_{-i}} \big) 
		+ \big( \widehat{V}_{i,h} - \overline{V}_{i,h}^{\pihat} \big) 
		+ \big( \overline{V}_{i,h}^{\pihat} - V_{i,h}^{\pihat} \big).  
		\label{eq:key_decomposition-further}
\end{align}

\paragraph{Step 2: establishing a key recursion.}

Recall the definition of $\pi_{h}^k$ in \eqref{eq:defn-pi-h-k-123}. 
Let us define the following auxiliary reward vectors $r_{i,h}^{\pihat}, r_{i,h}^{\pihatstar_{i} \times \pihat_{-i}}, \overline{r}_{i,h}\in \mathbb{R}^S$ 
as well as the auxiliary probability transition matrices $P_{i,h}^{\pihat},P_{i,h}^{\pihatstar_{i} \times \pihat_{-i}},\overline{P}_{i,h} \in \mathbb{R}^{S\times S}$ such that: 
for any $s,s'\in \cS$, 
\begin{subequations}
	\label{eq:defn-rh-Ph-mu-nu}
\begin{align}
	r_{i,h}^{\pihat}(s) &\coloneqq \sum_{k=1}^{K}\alpha_{k}^{K} \mathop{\mathbb{E}}_{\bm{a}\sim \pi_{h}^k(s)} \big[ r_{i,h}(s,\bm{a}) \big] 
	,  \\
	P_{i,h}^{\pihat}(s, s') &\coloneqq 
	\sum_{k=1}^{K}\alpha_{k}^{K} \mathop{\mathbb{E}}_{\bm{a} \sim \pi_{h}^k(s)} \big[ P_{h}(s'\mymid s,\bm{a}) \big] ,\\
	r_{i,h}^{\pihatstar_{i} \times \pihat_{-i}}(s) &\coloneqq \sum_{k=1}^{K}\alpha_{k}^{K} \mathop{\mathbb{E}}_{(a_i,\bm{a}_{-i})\sim \pihatstar_{i,h}(s)\times \pi_{-i,h}^k (s) } \big[ r_{i,h}(s,\bm{a}) \big] 
	,  \\
	P_{i,h}^{\pihatstar_{i} \times \pihat_{-i}}(s, s') &\coloneqq 
	\sum_{k=1}^{K}\alpha_{k}^{K} \mathop{\mathbb{E}}_{(a_i, \bm{a}_{-i})\sim \pihatstar_{i,h} (s)\times \pi_{-i,h}^k (s) } \big[ P_{h}(s'\mymid s, \bm{a}) \big] ,\\
		\overline{r}_{i,h}(s) &\coloneqq \sum_{k=1}^{K}\alpha_{k}^{K}\sum_{a_i\in \cA_i}\pi_{i,h}^{k}(a_i\mymid s) r_{i,h}^{k}(s,a_i), \\
    \overline{P}_{i,h}(s,s') &\coloneqq \sum_{k=1}^{K}\alpha_{k}^{K}\sum_{a_i\in \cA_i}\pi_{i,h}^{k}(a_i\mymid s) P_{i,h}^{k}(s'\mymid s,a_i) .
    \label{eq:defn-overline-P-hs}
\end{align}
\end{subequations}
As it turns out, $\overline{V}_{i,h}^{\pihat}$ (resp.~$\overline{V}_{i,h}^{\pihatstar_{i}\times\pihat_{-i}}$, $\widehat{V}_{i,h}$) stays reasonably close to the ``one-step-look-ahead'' expression $r_{i,h}^{\pihat} + P_{i,h}^{\pihat}\overline{V}_{i,h+1}^{\pihat}$ (resp.~$r_{i,h}^{\pihatstar_{i}\times\pihat_{-i}}+P_{i,h}^{\pihatstar_{i}\times\pihat_{-i}}\overline{V}_{i,h+1}^{\pihatstar_{i}\times\pihat_{-i}}$, 
$\overline{r}_{i,h} + \overline{P}_{i,h}\widehat{V}_{i,h+1}$), as revealed by the recursive relations stated in the following lemma; the proof of this lemma is deferred to Appendix~\ref{sec:proof-lemma-V-upper}. 
\begin{lemma}
\label{lem:V-upper}
There exists some universal constant $c_3>0$ such that with probability exceeding $1-\delta$,
\begin{subequations}
\label{eq:V-upper}
\begin{align}
	& \Big|\overline{V}_{i,h}^{\pihat}-\big(r_{i,h}^{\pihat}+P_{i,h}^{\pihat}\overline{V}_{i,h+1}^{\pihat}\big)\Big| \leq c_{3}\sqrt{\frac{H\log^{3}\big(\frac{KS\sum_i A_i }{\delta}\big)}{K}}1\notag\\
	& \qquad \qquad +c_{3}\sqrt{\frac{\log^{3}\big(\frac{KS\sum_i A_i}{\delta}\big)}{KH} }\Big[P_{i,h}^{\pihat}\big(\overline{V}_{i,h+1}^{\pihat}\circ\overline{V}_{i,h+1}^{\pihat}\big)-\big(P_{i,h}^{\pihat}\overline{V}_{i,h+1}^{\pihat}\big)\circ\big(P_{i,h}^{\pihat}\overline{V}_{i,h+1}^{\pihat}\big)\Big],
	\label{eq:V-upper-12}\\
	& \Big|\overline{V}_{i,h}^{\pihatstar_{i}\times \pihat_{-i}}-\big(r_{i,h}^{\pihatstar_{i}\times \pihat_{-i}}+P_{i,h}^{\pihatstar_{i}\times \pihat_{-i}}\overline{V}_{i,h+1}^{\pihatstar_{i}\times \pihat_{-i}}\big)\Big| \leq c_{3}\sqrt{\frac{H\log^{3}\big( \frac{KS \sum_i A_i }{\delta} \big) }{K}}1\notag\\
	& \qquad \qquad +c_{3}\sqrt{\frac{\log^{3} \big( \frac{KS \sum_i A_i }{\delta} \big) }{KH}}\Big[P_{i,h}^{\pihatstar_{i}\times \pihat_{-i}}\big(\overline{V}_{i,h+1}^{\pihatstar_{i}\times \pihat_{-i}}\circ\overline{V}_{i,h+1}^{\pihatstar_{i}\times \pihat_{-i}}\big)-\big(P_{i,h}^{\pihatstar_{i}\times \pihat_{-i}}\overline{V}_{i,h+1}^{\pihatstar_{i}\times \pihat_{-i}}\big)\circ\big(P_{i,h}^{\pihatstar_{i}\times \pihat_{-i}}\overline{V}_{i,h+1}^{\pihatstar_{i}\times \pihat_{-i}}\big)\Big],
	\label{eq:V-upper-34}\\
	& \Big|\widehat{V}_{i,h} - \big(\overline{r}_{i,h} + \overline{P}_{i,h}\widehat{V}_{i,h+1}\big)\Big| \leq c_{3}\sqrt{\frac{H\log^{3}\big(\frac{KS\sum_i A_i}{\delta} \big)}{K}}1\notag\\
	& \qquad \qquad +c_{3}\sqrt{\frac{\log^{3}\big(\frac{KS\sum_i A_i}{\delta}\big)}{KH}}\Big[\overline{P}_{i,h}\big(\widehat{V}_{i,h+1} \circ \widehat{V}_{i,h+1}\big)-\big(\overline{P}_{i,h}\widehat{V}_{i,h+1}\big) \circ \big(\overline{P}_{i,h}\widehat{V}_{i,h+1}\big)\Big]
	\label{eq:V-upper-56} 
\end{align}
\end{subequations}
hold for all $h\in [H]$. 
\end{lemma}
\begin{remark}
The right-hand side of each of the bounds in \eqref{eq:V-upper} contains a variance-style term (e.g., those terms taking the form of $P_{i,h}(V_{i,h+1}\circ V_{i,h+1})- (P_{i,h} V_{i,h+1}) \circ (P_{i,h} V_{i,h+1})$ for some probability transition matrix $P_{i,h}$ and value vector $V_{i,h+1}$).  Such variance-style terms are direct consequences of our Bernstein-style bonus terms, and are crucial in optimizing the horizon dependency. 
\end{remark}
With the above lemma in place, one can readily show that
\begin{align}
	\Big|\overline{V}_{i,h}^{\pihat}-P_{i,h}^{\pihat}\overline{V}_{i,h+1}^{\pihat}\Big| & \leq r_{i,h}^{\pihat}+c_{3}\sqrt{\frac{H\log^{3}\big( \frac{KS\sum_i A_i}{\delta} \big)}{K}}1\notag\\
	& \qquad+\frac{c_{3}}{H}\sqrt{\frac{H\log^{3}\big( \frac{KS\sum_i A_i}{\delta} \big)}{K}}\Big[P_{i,h}^{\pihat}\big(\overline{V}_{i,h+1}^{\pihat}\circ\overline{V}_{i,h+1}^{\pihat}\big)-\big(P_{i,h}^{\pihat}\overline{V}_{i,h+1}^{\pihat}\big)\circ\big(P_{i,h}^{\pihat}\overline{V}_{i,h+1}^{\pihat}\big)\Big]\notag\\
 & \leq\frac{c_{4}}{4}1+\frac{1}{4H}\Big[P_{i,h}^{\pihat}\big(\overline{V}_{i,h+1}^{\pihat}\circ\overline{V}_{i,h+1}^{\pihat}\big)-\big(P_{i,h}^{\pihat}\overline{V}_{i,h+1}^{\pihat}\big)\circ\big(P_{i,h}^{\pihat}\overline{V}_{i,h+1}^{\pihat}\big)\Big] \eqqcolon \zeta_0
	\label{eq:defn-zeta-0-Lemma4}
\end{align}
for some large enough constant $c_4 > 0$, 
where the last line holds due to Condition \eqref{eq:K-general-LB-135}, the basic fact 
$P_{i,h}^{\pihat}\big(\overline{V}_{i,h+1}^{\pihat}\circ\overline{V}_{i,h+1}^{\pihat}\big) \geq \big(P_{i,h}^{\pihat}\overline{V}_{i,h+1}^{\pihat}\big)\circ\big(P_{i,h}^{\pihat}\overline{V}_{i,h+1}^{\pihat}\big)$,  
 and the following fact (for large enough $c_4$)
\begin{align*}
	c_{3}\sqrt{\frac{H\log^{3}\big( \frac{KS\sum_i A_i}{\delta} \big) }{K}}1 + r_{i,h}^{\pihat} 
	\le c_{3}\sqrt{\frac{H\log^{3} \big( \frac{KS\sum_i A_i}{\delta} \big) }{K}}1 +  1 \le \frac{c_{4}}{4}1.
\end{align*}
In addition, recalling that $\|\overline{V}_{i,h}^{\pihat}\|_{\infty}, \|\overline{V}_{i,h+1}^{\pihat}\|_{\infty} \le H$ (cf.~\eqref{eq:V-max-output}) and recognizing that $\zeta_0\geq 0$ (see \eqref{eq:defn-zeta-0-Lemma4}), we can demonstrate that
\begin{align}
 & \Big|\overline{V}_{i,h}^{\pihat}\circ\overline{V}_{i,h}^{\pihat}-\big(P_{i,h}^{\pihat}\overline{V}_{i,h+1}^{\pihat}\big)\circ\big(P_{i,h}^{\pihat}\overline{V}_{i,h+1}^{\pihat}\big)\Big|=\Big|\big(\overline{V}_{i,h}^{\pihat}+P_{i,h}^{\pihat}\overline{V}_{i,h+1}^{\pihat}\big)\circ\big(\overline{V}_{i,h}^{\pihat}-P_{i,h}^{\pihat}\overline{V}_{i,h+1}^{\pihat}\big)\Big|\nonumber\\
 & \qquad\le\big(\overline{V}_{i,h}^{\pihat}+P_{i,h}^{\pihat}\overline{V}_{i,h+1}^{\pihat}\big)\circ\zeta_{0}  \leq2H\zeta_{0}\notag\\
 & \qquad=\frac{c_{4}}{2}H1+\frac{1}{2}\Big[P_{i,h}^{\pihat}\big(\overline{V}_{i,h+1}^{\pihat}\circ\overline{V}_{i,h+1}^{\pihat}\big)-\big(P_{i,h}^{\pihat}\overline{V}_{i,h+1}^{\pihat}\big)\circ\big(P_{i,h}^{\pihat}\overline{V}_{i,h+1}^{\pihat}\big)\Big].
	\label{eq:Vh-PVh-diff-UB135}
\end{align}
This further leads to
\begin{align*}
 & P_{i,h}^{\pihat}\big(\overline{V}_{i,h+1}^{\pihat}\circ\overline{V}_{i,h+1}^{\pihat}\big)-\big(P_{i,h}^{\pihat}\overline{V}_{i,h+1}^{\pihat}\big)\circ\big(P_{i,h}^{\pihat}\overline{V}_{i,h+1}^{\pihat}\big)\\
 & \quad=P_{i,h}^{\pihat}\big(\overline{V}_{i,h+1}^{\pihat}\circ\overline{V}_{i,h+1}^{\pihat}\big)-\overline{V}_{i,h}^{\pihat}\circ\overline{V}_{i,h}^{\pihat}+\overline{V}_{i,h}^{\pihat}\circ\overline{V}_{i,h}^{\pihat}-\big(P_{i,h}^{\pihat}\overline{V}_{i,h+1}^{\pihat}\big)\circ\big(P_{i,h}^{\pihat}\overline{V}_{i,h+1}^{\pihat}\big)\\
 & \quad\le P_{i,h}^{\pihat}\big(\overline{V}_{i,h+1}^{\pihat}\circ\overline{V}_{i,h+1}^{\pihat}\big)-\overline{V}_{i,h}^{\pihat}\circ\overline{V}_{i,h}^{\pihat}+\frac{c_{4}}{2}H1+\frac{1}{2}\Big[P_{i,h}^{\pihat}\big(\overline{V}_{i,h+1}^{\pihat}\circ\overline{V}_{i,h+1}^{\pihat}\big)-\big(P_{i,h}^{\pihat}\overline{V}_{i,h+1}^{\pihat}\big)\circ\big(P_{i,h}^{\pihat}\overline{V}_{i,h+1}^{\pihat}\big)\Big],
\end{align*}
which can be rearranged to yield
\begin{align*}
P_{i,h}^{\pihat}\big(\overline{V}_{i,h+1}^{\pihat}\circ\overline{V}_{i,h+1}^{\pihat}\big)-\big(P_{i,h}^{\pihat}\overline{V}_{i,h+1}^{\pihat}\big)\circ\big(P_{i,h}^{\pihat}\overline{V}_{i,h+1}^{\pihat}\big) & \le2\Big[P_{i,h}^{\pihat}\big(\overline{V}_{i,h+1}^{\pihat}\circ\overline{V}_{i,h+1}^{\pihat}\big)-\overline{V}_{i,h}^{\pihat}\circ\overline{V}_{i,h}^{\pihat}\Big]+c_{4}H1.
\end{align*}
Substituting it into \eqref{eq:V-upper-12} and combining terms give
\begin{align}
\Big|\overline{V}_{i,h}^{\pihat}-\big(r_{i,h}^{\pihat}+P_{i,h}^{\pihat}\overline{V}_{i,h+1}^{\pihat}\big) \Big|
	&\leq c_{5}\sqrt{\frac{H\log^{3} \big( \frac{KS\sum_i A_i}{\delta} \big) }{K}}1 \notag\\
	&\qquad+2c_{3}\sqrt{\frac{\log^{3} \big( \frac{KS\sum_i A_i}{\delta} \big) }{KH}}\Big[P_{i,h}^{\pihat}\big(\overline{V}_{i,h+1}^{\pihat}\circ\overline{V}_{i,h+1}^{\pihat}\big)-\overline{V}_{i,h}^{\pihat}\circ\overline{V}_{i,h}^{\pihat}\Big], \label{eq:Vh-Vh-UB-13579}
\end{align}
where we take $c_5=c_3+c_3c_4$. 

An analogous argument (which is omitted here for brevity) also reveals that
\begin{align}
	& \Big|\overline{V}_{i,h}^{\pihatstar_{i}\times\pihat_{-i}}-\big(r_{i,h}^{\pihatstar_{i}\times\pihat_{-i}}+P_{i,h}^{\pihatstar_{i}\times\pihat_{-i}}\overline{V}_{i,h+1}^{\pihatstar_{i}\times\pihat_{-i}}\big) \Big|
	\notag\\
	&\quad \leq c_{5}\sqrt{\frac{H\log^{3} \big( \frac{KS\sum_i A_i}{\delta} \big)}{K}}1 
	+2c_{3}\sqrt{\frac{\log^{3} \big( \frac{KS\sum_i A_i}{\delta} \big) }{KH}}\Big[P_{i,h}^{\pihatstar_{i}\times\pihat_{-i}}\big(\overline{V}_{i,h+1}^{\pihatstar_{i}\times\pihat_{-i}}\circ\overline{V}_{i,h+1}^{\pihatstar_{i}\mycrosstwo\pihat_{-i}}\big)-\overline{V}_{i,h}^{\pihatstar_{i}\times\pihat_{-i}}\circ\overline{V}_{i,h}^{\pihatstar_{i}\times\pihat_{-i}}\Big], \label{eq:Vstar-Bellman}\\
	&\Big|\widehat{V}_{i,h} - \big(\overline{r}_{i,h} + \overline{P}_{i,h}\widehat{V}_{i,h+1}\big)\Big| \notag\\
	&\quad \leq c_{5}\sqrt{\frac{H\log^{3} \big( \frac{KS\sum_i A_i}{\delta} \big)}{K}}1 
	+2c_{3}\sqrt{\frac{\log^{3} \big( \frac{KS\sum_i A_i}{\delta} \big)}{KH}}\Big[\overline{P}_{i,h}\big(\widehat{V}_{i,h+1}\circ\widehat{V}_{i,h+1}\big)-\widehat{V}_{i,h}\circ\widehat{V}_{i,h}\Big]. \label{eq:overlineV-Bellman}
\end{align}

\paragraph{Step 3: invoking the key recursion to establish the desired bound.}
We find it helpful to introduce the following notation (please note the order of the matrix product)
\[
\prod_{j:j<h}P_{i,j}^{\pihat} \coloneqq \begin{cases}
P_{i,1}^{\pihat}\cdots P_{i,h-1}^{\pihat}, & \text{if }h >1,\\
I, & \text{if }h=1.
\end{cases}
\]
Armed with this notation, 
we can invoke the relation \eqref{eq:Vh-Vh-UB-13579} recursively and use $\overline{V}_{i,h+1}^{\pihat}={V}_{i,h+1}^{\pihat}=0$ to obtain 
\begin{align}
\overline{V}_{i,h}^{\pihat} - V_{i,h}^{\pihat} 
	&\overset{\mathrm{(i)}}{=} r_{i,h}^{\widehat{\pi}} + P_{i,h}^{\pihat} \overline{V}_{i,h+1}^{\pihat} 
	+ \Big( \overline{V}_{i,h}^{\pihat}-\big(r_{i,h}^{\pihat}+P_{i,h}^{\pihat}\overline{V}_{i,h+1}^{\pihat}\big) \Big) 
	- \big( r_{i,h}^{\widehat{\pi}} + P_{i,h}^{\pihat} V_{i,h+1}^{\pihat} \big) \notag\\
	&\le P_{i,h}^{\pihat}\big(\overline{V}_{i,h+1}^{\pihat} - V_{i,h+1}^{\pihat}\big) + \Big|\overline{V}_{i,h}^{\pihat}-\big(r_{i,h}^{\pihat}+P_{i,h}^{\pihat}\overline{V}_{i,h+1}^{\pihat}\big) \Big| \label{eq:intermediate-recursive-1357}\\
	& \overset{\mathrm{(ii)}}{\le} c_{5}\sqrt{\frac{H\log^{3}\big( \frac{KS\sum_i A_i}{\delta} \big)}{K}}\left(\sum_{h=1}^{H}\prod_{j:j<h}P_{i,j}^{\pihat}\right)1\notag\\
	& \qquad +2c_{3}\sqrt{\frac{\log^{3}\big(\frac{KS\sum_i A_i}{\delta}\big)}{KH}}\sum_{h=1}^{H}\prod_{j:j<h}P_{i,j}^{\pihat}\Big[P_{i,h}^{\pihat}\big(\overline{V}_{i,h+1}^{\pihat}\circ\overline{V}_{i,h+1}^{\pihat}\big)-\overline{V}_{i,h}^{\pihat}\circ\overline{V}_{i,h}^{\pihat}\Big]\notag\\
	& \overset{\mathrm{(iii)}}{\le}  c_{5}\sqrt{\frac{H\log^{3}\big(\frac{KS\sum_i A_i}{\delta}\big)}{K}}\left(\sum_{h=1}^{H}\prod_{j:j<h}P_{i,j}^{\pihat}\right)1 
	= c_{5}\sqrt{\frac{H^{3}\log^{3}\big(\frac{KS\sum_i A_i}{\delta}\big)}{K}}1\leq\frac{\varepsilon}{3}1.
	\label{eq:Vh-Vh-UB-2468}
\end{align}
Here, (i) uses the Bellman equation; (ii) applies the bound \eqref{eq:Vh-Vh-UB-13579} recursively;  
(iii) holds since for any transition matrices $\{P_{i,h}\}$ and any sequence $\{V_{i,h}\}$ obeying $V_{i,h+1}=0$, one can use the telescoping sum to obtain
\begin{align*}
\sum_{h=1}^{H}\prod_{j:j<h}P_{i,j}\Big[P_{i,h}\big(V_{i,h+1}\circ V_{i,h+1}\big)-V_{i,h}\circ V_{i,h}\Big] & =\sum_{h=1}^{H}\prod_{j:j\le h}P_{i,j}\big(V_{i,h+1}\circ V_{i,h+1}\big)-\sum_{h=1}^{H}\prod_{j:j<h}P_{i,j}\big(V_{i,h}\circ V_{i,h}\big)\\
 & =\prod_{j:j\le H}P_{i,j}\big(V_{i,h+1}\circ V_{i,h+1}\big)-V_{i,1}\circ V_{i,1}\\
 & =-V_{i,1}\circ V_{i,1}\leq0,
\end{align*}
whereas the last inequality in \eqref{eq:Vh-Vh-UB-2468} arises from the assumption \eqref{eq:K-sample-size-thm} when $c_{\mathsf{k}}$ is large enough. 
Similarly, replacing $\pihat_{i}$ with $\pihatstar_{i}$ in the above argument and recalling~\eqref{eq:Vstar-Bellman} directly lead to
\begin{align} \label{eq:two_additional-1}
V_{i,h}^{\star,\pihat_{-i}} - \overline{V}_{i,h}^{\pihatstar_{i}\times\pihat_{-i}} 
	= V_{i,h}^{\pihatstar_{i} \times \pihat_{-i}} - \overline{V}_{i,h}^{\pihatstar_{i}\times\pihat_{-i}}  \leq\frac{\varepsilon}{3}1.
\end{align}

In addition, recalling the definition of $\overline{V}_{i,h}^{\pihat}$ (cf.~\eqref{defi:V-muhat}), $\overline{r}_{i,h}$ and $\overline{P}_{i,h}$ (see \eqref{eq:defn-rh-Ph-mu-nu}), we can deduce that
\begin{align*}
\widehat{V}_{i,h}-\overline{V}_{i,h}^{\pihat} & =\overline{r}_{i,h}+\overline{P}_{i,h}\widehat{V}_{i,h+1}+\Big\{\widehat{V}_{i,h}-\big(\overline{r}_{i,h}+\overline{P}_{i,h}\widehat{V}_{i,h+1}\big)\Big\}-\overline{r}_{i,h}-\overline{P}_{i,h}\overline{V}_{i,h+1}^{\pihat}\\
 & \leq\overline{P}_{i,h}\big(\widehat{V}_{i,h+1}-\overline{V}_{i,h+1}^{\pihat}\big)+\Big|\widehat{V}_{i,h}-\big(\overline{r}_{i,h}+\overline{P}_{i,h}\widehat{V}_{i,h+1}\big)\Big|, 
\end{align*}
which resembles \eqref{eq:intermediate-recursive-1357}. 
Thus, repeating the above argument for \eqref{eq:Vh-Vh-UB-2468} and applying~\eqref{eq:overlineV-Bellman} recursively, 
we reach
\begin{align}
	\widehat{V}_{i,h} - \overline{V}_{i,h}^{\pihat} 
	\leq\frac{\varepsilon}{3}1. \label{eq:two_additional-2}
\end{align}

To finish up, combining \eqref{eq:Vh-Vh-UB-2468}, \eqref{eq:two_additional-1} and \eqref{eq:two_additional-2} with \eqref{eq:key_decomposition-further}, we arrive at
\begin{align*}
	V_{i,h}^{\star,\pihat_{-i}}-V_{i,h}^{\pihat} 
	\le \big( V_{i,h}^{\star,\pihat_{-i}} - \overline{V}_{i,h}^{\pihatstar_{i}\times\pihat_{-i}} \big) 
	+ \big( \widehat{V}_{i,h} - \overline{V}_{i,h}^{\pihat} \big) + \big( \overline{V}_{i,h}^{\pihat} - V_{i,h}^{\pihat} \big) \leq \varepsilon 1. 
\end{align*}
This establishes the first inequality in \eqref{eq:target-two-inequalities}, while the second inequality in \eqref{eq:target-two-inequalities} can be validated via the same argument. 
We have thus completed the proof of Theorem~\ref{eq:main-result-zero-sum}.

\section{Discussion}
\label{sec:discussion}

The primary contribution of this paper has been to develop a sample-optimal paradigm that simultaneously overcomes the curse of multiple agents and optimizes the horizon dependency when solving multi-player Markov games. This goal was not accomplished in any of the previous works, regardless of the sampling mechanism in use.  
The adoption of the adversarial learning subroutine helps break the curse of multiple agents compared to the prior model-based approach \citep{zhang2020marl,liu2021sharp}, 
whereas the availability of the generative model in conjunction with the variance-aware bonus design improves horizon dependency compared to \citet{bai2020near,jin2021v}.  
 Our work opens further questions surrounding sample efficiency in solving Markov games. 
For instance, 
our sample complexity bound \eqref{eq:sample-complexity-Thm-1} is likely suboptimal (by a factor of, say, $\frac{\sum_i A_i}{\max_i A_i}$) when the number $m$ of players is allowed to grow with other parameters; 
can we further optimize this via a more refined learning algorithm?  
Also, how to attain minimax-optimal sample complexity if we only have access to less idealistic sampling protocol (e.g., local access models \citep{li2021sample,yin2022efficient}, and online sampling protocols \citep{azar2017minimax,jin2018q}) as opposed to the flexible generative model? How can we optimize the horizon dependency when computing {\em correlated equilibria (CE)} in multi-agent general-sum scenarios \citep{song2021can,jin2021v} without compromising the dependency on the size of the action spaces.    In addition, our refined regret bound for FTRL (based on variance-type quantities) only covers the full-information case; 
it would be of interest to generalize it to the bandit-feedback setting (where only partial entries of the loss vectors are observable each time). 



\section*{Acknowledgements}

Y.~Chen is supported in part by the Alfred P.~Sloan Research Fellowship, the Google Research Scholar Award, the AFOSR grant FA9550-22-1-0198, 
the ONR grant N00014-22-1-2354,  and the NSF grants CCF-2221009, CCF-1907661, 
IIS-2218713 and IIS-2218773. 
Y.~Wei is supported in part by the the NSF grants CCF-2106778, DMS-2147546/2015447 and  CAREER award DMS-2143215. 
Y.~Chi are supported in part by the grants ONR N00014-19-1-2404, NSF CCF-2106778 and DMS-2134080, and CAREER award ECCS-1818571. 
Part of this work was done while G.~Li, Y.~Wei and Y.~Chen were visiting the Simons Institute for the Theory of Computing.

\appendix

\section{Proof of Theorem~\ref{thm:FTRL-refined}}
\label{sec:proof-thm:FTRL-refined}

This section is devoted to presenting the proof of Theorem~\ref{thm:FTRL-refined}. 
Before embarking on the analysis, let us introduce a convenient auxiliary iterate  
\begin{align}
	\pi_{k+1}^{-} = \arg\min_{\pi\in \Delta(\cA)} \, \bigg\{ \langle \pi, L_k\rangle + \frac{1}{\widehat{\eta}_{k}}F(\pi) \bigg\} ,
	\label{eq:pi-k-minus}
\end{align}
or equivalently, 
\begin{align}
	\pi_{k+1}^{-} (a) =
	\frac{ \exp\big( - \widehat{\eta}_{k}L_{k} (a) \big) }{  \sum_{a'\in \cA} \exp\big(- \widehat{\eta}_{k}L_{k}(a') \big) }
	\qquad \text{for all } a\in \cA, 
\end{align}
which differs from \eqref{eq:FTRL-update-entropy} only in the learning rates being used (namely, $\pi_{k+1}$ uses $\eta_{k+1}$ while $\pi_{k+1}^-$ adopts $\widehat{\eta}_k$).

\subsection{Main steps of the proof}

The key steps of the proof lie in justifying the following two claims: 
\begin{align}
	R_{n} & \le\sum_{k=1}^{n}\alpha_{k}^{n}\big\langle\pi_{k}-\pi_{k+1}^{-},\ell_{k}\big\rangle + \frac{\log A}{\eta_{n+1}}; 
	\label{eq:general-bound-Rn-135}
\end{align}
%
and for all $a\in \cA$ and all $k\geq 1$, 
\begin{align} \label{eq:pi-minus-LB}
\pi_{k+1}^{-}(a) \ge 
	\begin{cases}
		\big[1 - \widehat{\eta}_{k}\alpha_{k} \ell_k(a)\big]\pi_{k}(a), & \text{if } \widehat{\eta}_{k}\alpha_{k} \| \ell_k \|_{\infty} > \frac{1}{3}, \\ 
		\Big\{1-\widehat{\eta}_{k}\alpha_{k}\big(\ell_{k}(a)-\mathbb{E}_{\pi_{k}}[\ell_{k}]\big)-2\widehat{\eta}_{k}^2\alpha_{k}^2\mathsf{Var}_{\pi_{k}}\big(\ell_{k}\big)\Big\} \pi_{k}(a),  & \text{if } \widehat{\eta}_{k}\alpha_{k} \| \ell_k \|_{\infty} \le \frac{1}{3}, 
	\end{cases}
\end{align}
where for any vector $\ell\in \mathbb{R}^A$ we define 
\[
	\mathbb{E}_{\pi_k}[\ell] \coloneqq \sum_{a\in \cA}\pi_k(a)\ell(a).
\]
In words, the first claim \eqref{eq:general-bound-Rn-135} allows us to replace the action that appears best in hindsight (cf.~\eqref{eq:regret-defn-FTRL}) 
by the time-varying predictions $\{\pi_{k+1}^-\}$ without incurring much cost, 
whereas the second claim \eqref{eq:pi-minus-LB} controls the proximity of $\pi_{k+1}^-$ and $\pi_{k}$ in each round. 
Let us assume the validity of these two claims for the moment, and return to prove them shortly.


In view of the upper bound \eqref{eq:general-bound-Rn-135}, 
we are in need of controlling $\big\langle\pi_{k}-\pi_{k+1}^{-},\ell_{k}\big\rangle$.   
We divide into two cases.
\begin{itemize}
\item
For any $k$ obeying $\widehat{\eta}_{k}\alpha_{k} \| \ell_k \|_{\infty} > 1/3$, invoke \eqref{eq:pi-minus-LB} and the non-negativity of $\ell_k$ to reach
\begin{align}
	\big\langle\pi_{k}-\pi_{k+1}^{-},\ell_{k}\big\rangle & \leq\sum_{a\in\cA}\widehat{\eta}_{k}\alpha_{k}\pi_{k}(a)\big[\ell_{k}(a)\big]^{2}
	\leq \sum_{a\in\cA}\widehat{\eta}_{k}\alpha_{k}\pi_{k}(a)\big\|\ell_{k}\big\|_{\infty}^{2}
	= \widehat{\eta}_{k}\alpha_{k}\big\|\ell_{k}\big\|_{\infty}^{2}. 
	\label{eq:pi-piminus-diff-bound-large}
\end{align}
\item
In contrast, if $\widehat{\eta}_{k}\alpha_{k} \| \ell_k \|_{\infty} \leq 1/3$, then it follows from \eqref{eq:pi-minus-LB} that
\begin{align}
\big\langle\pi_{k}-\pi_{k+1}^{-},\ell_{k}\big\rangle & \leq\sum_{a\in\cA}\bigg\{\widehat{\eta}_{k}\alpha_{k}\big(\ell_{k}(a)-\mathbb{E}_{\pi_{k}}[\ell_{k}]\big)+2\widehat{\eta}_{k}^2\alpha_{k}^2\mathsf{Var}_{\pi_{k}}(\ell_{k})\bigg\}\pi_{k}(a)\ell_{k}(a)\nonumber \\
 & =\widehat{\eta}_{k}\alpha_{k}\sum_{a\in\cA}\pi_{k}(a)\big(\ell_{k}(a)-\mathbb{E}_{\pi_{k}}\big[\ell_{k}\big]\big) \mathbb{E}_{\pi_{k}}\big[\ell_{k}\big] +\widehat{\eta}_{k}\alpha_{k}\sum_{a\in\cA}\pi_{k}(a)\big(\ell_{k}(a)-\mathbb{E}_{\pi_{k}}\big[\ell_{k}\big]\big)^{2}\nonumber \\
 & \qquad+2\widehat{\eta}_{k}^2\alpha_{k}^2\mathsf{Var}_{\pi_{k}}(\ell_{k})\sum_{a\in\cA}\pi_{k}(a)\ell_{k}(a)\nonumber \\
 & =\widehat{\eta}_{k}\alpha_{k}\sum_{a\in\cA}\pi_{k}(a)\big(\ell_{k}(a)-\mathbb{E}_{\pi_{k}}\big[\ell_{k}\big]\big)^{2}+2\widehat{\eta}_{k}^2\alpha_{k}^2\mathsf{Var}_{\pi_{k}}(\ell_{k})\sum_{a\in\cA}\pi_{k}(a)\ell_{k}(a)\nonumber \\
 & \leq\widehat{\eta}_{k}\alpha_{k}\mathsf{Var}_{\pi_{k}}\big(\ell_{k}\big)+2\widehat{\eta}_{k}^2\alpha_{k}^2\mathsf{Var}_{\pi_{k}}(\ell_{k})\big\|\ell_{k}\big\|_{\infty},\label{eq:pik-piminus-UB-234}
\end{align}
where we invoke the elementary facts that $\sum_{a}\pi_{k}(a)\big(\ell_{k}(a)-\mathbb{E}_{\pi_{k}}\big[\ell_{k}\big]\big)=0$ and $\sum_{a}\pi_{k}(a)\ell_{k}(a)\leq \|\ell_{k}\|_{\infty}$. 
\end{itemize}

Putting the above two cases together yields
\begin{align}
 & \sum_{k=1}^{n}\alpha_{k}^{n}\big\langle\pi_{k}-\pi_{k+1}^{-},\ell_{k}\big\rangle \notag\\
 & \quad\leq\sum_{k=1}^{n}\alpha_{k}^{n}\widehat{\eta}_{k}\alpha_{k}\big\|\ell_{k}\big\|_{\infty}^{2}\ind\bigg(\widehat{\eta}_{k}\alpha_{k}\big\|\ell_{k}\big\|_{\infty}>\frac{1}{3}\bigg)+\sum_{k=1}^{n}\alpha_{k}^{n}\widehat{\eta}_{k}\alpha_{k}\mathsf{Var}_{\pi_{k}}\big(\ell_{k}\big)\ind\bigg(\widehat{\eta}_{k}\alpha_{k}\big\|\ell_{k}\big\|_{\infty}\leq\frac{1}{3}\bigg)\nonumber\\
 & \qquad+2\sum_{k=1}^{n}\alpha_{k}^{n}\widehat{\eta}_{k}^2\alpha_{k}^2\mathsf{Var}_{\pi_{k}}\big(\ell_{k}\big)\big\|\ell_{k}\big\|_{\infty}\ind\bigg(\widehat{\eta}_{k}\alpha_{k}\big\|\ell_{k}\big\|_{\infty}\leq\frac{1}{3}\bigg)\nonumber\\
 & \quad\leq\frac{5}{3}\sum_{k=1}^{n}\alpha_{k}^{n}\widehat{\eta}_{k}\alpha_{k}\mathsf{Var}_{\pi_{k}}\big(\ell_{k}\big)
    + 3\sum_{k=1}^{n}\alpha_{k}^{n}\widehat{\eta}_{k}^2\alpha_{k}^2\big\|\ell_{k}\big\|_{\infty}^{3}\ind\bigg(\widehat{\eta}_{k}\alpha_{k}\big\|\ell_{k}\big\|_{\infty}>\frac{1}{3}\bigg),\label{eq:pik-piminus-UB-456}
\end{align}
where the last inequality holds true since 
\begin{align*}
\sum_{k=1}^{n}\alpha_{k}^{n}\widehat{\eta}_{k}\alpha_{k}\big\|\ell_{k}\big\|_{\infty}^2\ind\bigg(\widehat{\eta}_{k}\alpha_{k}\big\|\ell_{k}\big\|_{\infty}>\frac{1}{3}\bigg) & \le3\sum_{k=1}^{n}\alpha_{k}^{n}\widehat{\eta}_{k}^2\alpha_{k}^2\big\|\ell_{k}\big\|_{\infty}^{3}\ind\bigg(\widehat{\eta}_{k}\alpha_{k}\big\|\ell_{k}\big\|_{\infty}>\frac{1}{3}\bigg),\\
\sum_{k=1}^{n}\alpha_{k}^{n}\widehat{\eta}_{k}^2\alpha_{k}^2\big\|\ell_{k}\big\|_{\infty}\mathsf{Var}_{\pi_{k}}(\ell_{k})\ind\bigg(\widehat{\eta}_{k}\alpha_{k}\big\|\ell_{k}\big\|_{\infty}\le\frac{1}{3}\bigg) & \le\frac{1}{3}\sum_{k=1}^{n}\alpha_{k}^{n}\widehat{\eta}_{k}\alpha_{k}\mathsf{Var}_{\pi_{k}}(\ell_{k}).
\end{align*}
%

%

Substituting \eqref{eq:pik-piminus-UB-456} into \eqref{eq:general-bound-Rn-135},  we can readily arrive at
\begin{align*}
R_{n} 
 & \le\frac{5}{3}\sum_{k=1}^{n}\alpha_{k}^{n}\widehat{\eta}_{k}\alpha_{k}\mathsf{Var}_{\pi_{k}}(\ell_{k})+\frac{\log A}{\eta_{n+1}}+3\sum_{k=1}^{n}\alpha_{k}^{n}\widehat{\eta}_{k}^2\alpha_{k}^2\big\|\ell_{k}\big\|_{\infty}^{3}\ind\bigg(\widehat{\eta}_{k}\alpha_{k}\big\|\ell_{k}\big\|_{\infty}>\frac{1}{3}\bigg).
\end{align*}
It thus remains to establish the claims~\eqref{eq:general-bound-Rn-135} and \eqref{eq:pi-minus-LB}, which we shall accomplish next.

%


\subsection{Proof of claim~\eqref{eq:general-bound-Rn-135}}
We claim that it suffices to prove that 
\begin{align}
	&\alpha_{1}^{n}\langle\pi_{2}^{-},\ell_{1}\rangle+\frac{\alpha_{1}^{n}}{\eta_{2}\alpha_{1}}F(\pi_{2}) + \sum_{k=2}^{n}\bigg\{\alpha_{k}^{n}\langle\pi_{k+1}^{-},\ell_{k}\rangle+\Big[\frac{\alpha_{k}^{n}}{\eta_{k+1}\alpha_{k}}-\frac{\alpha_{k-1}^{n}}{\eta_{k}\alpha_{k-1}}\Big]F(\pi_{k+1})\bigg\} \notag\\
	&\qquad\qquad\le\min_{\pi\in\Delta(\cA)}\left\{ \bigg\langle\pi,\sum_{k=1}^{n}\alpha_{k}^{n}\ell_{k}\bigg\rangle+\frac{1}{\eta_{n+1}}F(\pi)\right\} .
	\label{eq:weighted-alpha-1357}
\end{align}
In fact, suppose that this inequality \eqref{eq:weighted-alpha-1357} is valid, then one can easily obtain
\begin{align*}
 &  	\alpha_{1}^{n}\langle\pi_{2}^{-},\ell_{1}\rangle+\frac{\alpha_{1}^{n}}{\eta_{2}\alpha_{1}}F(\pi_{2}) + \sum_{k=2}^{n}\bigg\{\alpha_{k}^{n}\langle\pi_{k+1}^{-},\ell_{k}\rangle+\Big[\frac{\alpha_{k}^{n}}{\eta_{k+1}\alpha_{k}}-\frac{\alpha_{k-1}^{n}}{\eta_{k}\alpha_{k-1}}\Big]F(\pi_{k+1})\bigg\} \\
 & \qquad\le\min_{\pi\in\Delta(\cA)}\left\{ \bigg\langle\pi,\sum_{k=1}^{n}\alpha_{k}^{n}\ell_{k}\bigg\rangle+\frac{1}{\eta_{n+1}}F(\pi)\right\} \le\min_{\pi\in\{e_{a}\mymid a\in\cA\}}\left\{ \bigg\langle\pi,\sum_{k=1}^{n}\alpha_{k}^{n}\ell_{k}\bigg\rangle+\frac{1}{\eta_{n+1}}F(\pi)\right\} \\
 & \qquad=\min_{\pi\in\{e_{a}\mymid a\in\cA\}}\bigg\langle\pi,\sum_{k=1}^{n}\alpha_{k}^{n}\ell_{k}\bigg\rangle=\min_{a\in\cA}\sum_{k=1}^{n}\alpha_{k}^{n}\ell_{k}(a)
\end{align*}
with $e_a$ the $a$-th standard basis vector in $\mathbb{R}^A$, 
where the last line holds true since the negative entropy obeys $F(e_a)=0$ for any $a\in \cA$. 
In turn, this implies that
\begin{align}
R_{n} & =\sum_{k=1}^{n}\alpha_{k}^{n}\big\langle\pi_{k},\ell_{k}\big\rangle-\min_{a\in\cA}\sum_{k=1}^{n}\alpha_{k}^{n}\ell_{k}(a) \notag\\
	& \leq\sum_{k=1}^{n}\alpha_{k}^{n}\big\langle\pi_{k}-\pi_{k+1}^{-},\ell_{k}\big\rangle-\sum_{k=2}^{n}\Big[\frac{\alpha_{k}^{n}}{\eta_{k+1}\alpha_{k}}-\frac{\alpha_{k-1}^{n}}{\eta_{k}\alpha_{k-1}}\Big]F(\pi_{k+1}) 
	+ \frac{\alpha_1^n}{\eta_{2}\alpha_1}\log A
	, \label{eq:Rn-bound-135}
\end{align}
where the last inequality invokes the elementary fact $-F(\pi)\leq \log A$ for any $\pi\in \Delta(\cA)$. 
Additionally, under the assumptions that $\eta_{k+1}(1-\alpha_k)\leq \eta_k$ ($k\geq 1$), we can use  the definition \eqref{def:alpha-i-k}  to obtain
\[
\frac{\alpha_{k}^{n}}{\eta_{k+1}\alpha_{k}}=\frac{\prod_{j=k+1}^{n}(1-\alpha_{j})}{\eta_{k+1}}\geq\frac{\prod_{j=k}^{n}(1-\alpha_{j})}{\eta_{k}}=\frac{\alpha_{k-1}^{n}}{\eta_{k}\alpha_{k-1}} ,
\]
for any $k\geq 2$, which together with the basic fact $0\leq -F(\pi)\leq \log A$ yields
\begin{align}
-\sum_{k=2}^{n}\Big[\frac{\alpha_{k}^{n}}{\eta_{k+1}\alpha_{k}}-\frac{\alpha_{k-1}^{n}}{\eta_{k}\alpha_{k-1}}\Big]F(\pi_{k+1}) 
	+ \frac{\alpha_1^n}{\eta_{2}\alpha_1}\log A
	& \leq\sum_{k=2}^{n}\Big[\frac{\alpha_{k}^{n}}{\eta_{k+1}\alpha_{k}}-\frac{\alpha_{k-1}^{n}}{\eta_{k}\alpha_{k-1}}\Big]\log A + \frac{\alpha_1^n}{\eta_{2}\alpha_1}\log A\notag\\
 & =\frac{\alpha_{n}^{n}}{\eta_{n+1}\alpha_{n}}\log A=\frac{\log A}{\eta_{n+1}}. 
	\label{eq:sum-Fk-non-negative-456}
\end{align}
Substitution into \eqref{eq:Rn-bound-135} leads to
\begin{align}
R_{n} & \leq\sum_{k=1}^{n}\alpha_{k}^{n}\big\langle\pi_{k}-\pi_{k+1}^{-},\ell_{k}\big\rangle
 + \frac{\log A}{\eta_{n+1}}
\end{align}
as advertised.  As a consequence, everything boils down to establishing \eqref{eq:weighted-alpha-1357}.

Towards this end, we would like to proceed with an induction argument, with the induction hypothesis w.r.t.~$n$ given by~\eqref{eq:weighted-alpha-1357}. 
Firstly, the base case with $n=1$ simplifies to
\[
	\alpha_{1}^1\langle\pi_{2}^{-},\ell_{1}\rangle+\frac{1}{\eta_{2}}F(\pi_{2}) \le \min_{\pi\in\Delta(\cA)} \left\{ \langle \pi, \alpha_1^1\ell_1\rangle + \frac{1}{\eta_2}F(\pi)\right\}
\]
given that $\alpha_1=\alpha_1^1$; this inequality clearly holds since, according to \eqref{eq:FTRL-update-general} and \eqref{eq:pi-k-minus},  
\[
 \pi_{2}^{-} = \pi_{2}
 = \arg\min_{\pi\in\Delta(\cA)}\left\{ \langle\pi,L_{1}\rangle+\frac{1}{\eta_2}F(\pi)\right\}
 = \arg\min_{\pi\in\Delta(\cA)}\left\{ \langle\pi,\alpha_1\ell_{1}\rangle+\frac{1}{\eta_2}F(\pi)\right\}.
\]
Secondly,  suppose that \eqref{eq:weighted-alpha-1357} holds w.r.t.~$n$, 
and we intend to justify it w.r.t.~$n+1$. 
To do so, we observe that
\begin{align}
 & \alpha_{1}^{n+1}\langle\pi_{2}^{-},\ell_{1}\rangle+\frac{\alpha_{1}^{n+1}}{\eta_{2}\alpha_{1}}F(\pi_{2})+\sum_{k=2}^{n}\bigg\{\alpha_{k}^{n+1}\langle\pi_{k+1}^{-},\ell_{k}\rangle+\Big(\frac{\alpha_{k}^{n+1}}{\eta_{k+1}\alpha_{k}}-\frac{\alpha_{k-1}^{n+1}}{\eta_{k}\alpha_{k-1}}\Big)F(\pi_{k+1})\bigg\}+\alpha_{n+1}\langle\pi_{n+2}^{-},\ell_{n+1}\rangle\notag\\
 & \quad\overset{(\mathrm{i})}{=}(1-\alpha_{n+1})\left\{ \alpha_{1}^{n}\langle\pi_{2}^{-},\ell_{1}\rangle+\frac{\alpha_{1}^{n}}{\eta_{2}\alpha_{1}}F(\pi_{2})+\sum_{k=2}^{n}\bigg\{\alpha_{k}^{n}\langle\pi_{k+1}^{-},\ell_{k}\rangle+\Big(\frac{\alpha_{k}^{n}}{\eta_{k+1}\alpha_{k}}-\frac{\alpha_{k-1}^{n}}{\eta_{k}\alpha_{k-1}}\Big)F(\pi_{k+1})\bigg\}\right\} \notag\\
 &\notag \qquad\qquad+\alpha_{n+1}\langle\pi_{n+2}^{-},\ell_{n+1}\rangle\notag\\
 &\notag \quad\overset{\mathrm{(ii)}}{\leq}(1-\alpha_{n+1})\left\{ \bigg\langle\pi_{n+2}^{-},\sum_{k=1}^{n}\alpha_{k}^{n}\ell_{k}\bigg\rangle+\frac{1}{\eta_{n+1}}F(\pi_{n+2}^{-})\right\} +\alpha_{n+1}\langle\pi_{n+2}^{-},\ell_{n+1}\rangle\\
 & \quad\overset{\mathrm{(iii)}}{=}\bigg\langle\pi_{n+2}^{-},\sum_{k=1}^{n+1}\alpha_{k}^{n+1}\ell_{k}\bigg\rangle+\frac{1-\alpha_{n+1}}{\eta_{n+1}}F(\pi_{n+2}^{-})
	= \min_{\pi\in\Delta(\cA)}\left\{ \bigg\langle\pi,\sum_{k=1}^{n+1}\alpha_{k}^{n+1}\ell_{k}\bigg\rangle+\frac{1}{\widehat{\eta}_{n+1}}F(\pi)\right\}. 
	\label{eq:weighted-alpha-579}
\end{align}
%
Here, (i) and (iii) invoke the fact $\alpha_{k}^{n+1}=(1-\alpha_{n+1})\alpha_{k}^{n}$ and $\alpha_{n+1}^{n+1}=\alpha_{n+1}$
(according to (\ref{def:alpha-i-k})), (ii) relies on the induction
hypothesis (\ref{eq:weighted-alpha-1357}) w.r.t.~$n$. 
To finish up, invoke \eqref{eq:weighted-alpha-579} and the definition \eqref{def:alpha-i-k} to arrive at
\begin{align*}
 & \alpha_{1}^{n+1}\langle\pi_{2}^{-},\ell_{1}\rangle+\frac{\alpha_{1}^{n+1}}{\eta_{2}\alpha_{1}}F(\pi_{2}) + \sum_{k=2}^{n+1}\bigg\{\alpha_{k}^{n+1}\langle\pi_{k+1}^{-},\ell_{k}\rangle+\Big[\frac{\alpha_{k}^{n+1}}{\eta_{k+1}\alpha_{k}}-\frac{\alpha_{k-1}^{n+1}}{\eta_{k}\alpha_{k-1}}\Big]F(\pi_{k+1})\bigg\}\notag\\
 & =\left\{ \alpha_{1}^{n+1}\langle\pi_{2}^{-},\ell_{1}\rangle+\frac{\alpha_{1}^{n+1}}{\eta_{2}\alpha_{1}}F(\pi_{2})+\sum_{k=2}^{n}\bigg\{\alpha_{k}^{n+1}\langle\pi_{k+1}^{-},\ell_{k}\rangle+\Big[\frac{\alpha_{k}^{n+1}}{\eta_{k+1}\alpha_{k}}-\frac{\alpha_{k-1}^{n+1}}{\eta_{k}\alpha_{k-1}}\Big]F(\pi_{k+1})\bigg\}+\alpha_{n+1}\langle\pi_{n+2}^{-},\ell_{n+1}\rangle\right\} \notag\\
	& \qquad \qquad +\Big[\frac{1}{\eta_{n+2}}-\frac{1-\alpha_{n+1}}{\eta_{n+1}}\Big]F(\pi_{n+2})\notag\\
 & \le\left\{ \bigg\langle\pi_{n+2},\sum_{k=1}^{n+1}\alpha_{k}^{n+1}\ell_{k}\bigg\rangle+\frac{1-\alpha_{n+1}}{\eta_{n+1}}F(\pi_{n+2})\right\} +\Big[\frac{1}{\eta_{n+2}}-\frac{1-\alpha_{n+1}}{\eta_{n+1}}\Big]F(\pi_{n+2})\notag\\
 & =\bigg\langle\pi_{n+2},\sum_{k=1}^{n+1}\alpha_{k}^{n+1}\ell_{k}\bigg\rangle+\frac{1}{\eta_{n+2}}F(\pi_{n+2})
	=\min_{\pi\in\Delta(\cA)}\left\{ \bigg\langle\pi,\sum_{k=1}^{n+1}\alpha_{k}^{n+1}\ell_{k}\bigg\rangle+\frac{1}{\eta_{n+2}}F(\pi)\right\} ,
\end{align*}
where the inequality above makes use of  \eqref{eq:weighted-alpha-579}, 
and the last identity comes from \eqref{eq:FTRL-update-general}. This justifies the induction hypothesis w.r.t.~$n+1$. 
Applying the induction argument in turn establishes  \eqref{eq:weighted-alpha-1357} for all $n$, thereby concluding the proof.

\subsection{Proof of claim~\eqref{eq:pi-minus-LB}}
We first make the observation that 
\begin{align*}
\sum_{a}\exp\big(-\widehat{\eta}_{k}L_{k}(a)\big) & =\sum_{a}\exp\big(-\eta_{k}L_{k-1}(a)\big)\exp\big(-\widehat{\eta}_{k}\alpha_{k}\ell_{k}(a)\big)\\
 & =\sum_{a}\left\{ \pi_{k}(a)\sum_{a'}\exp\big(-\eta_{k}L_{k-1}(a')\big)\right\} \exp\big(-\widehat{\eta}_{k}\alpha_{k}\ell_{k}(a)\big)\\
 & =\sum_{a'}\exp\big(-\eta_{k}L_{k-1}(a')\big)\sum_{a}\left\{ \pi_{k}(a)\exp\big(-\widehat{\eta}_{k}\alpha_{k}\ell_{k}(a)\big)\right\} ,
\end{align*}
where the second equality follows from \eqref{eq:FTRL-update-entropy}. This in turn allows us to demonstrate that
\begin{align*}
\pi_{k+1}^{-}(a) & =\frac{\exp\big(-\widehat{\eta}_{k}L_{k}(a)\big)}{\sum_{a'}\exp\big(-\widehat{\eta}_{k}L_{k}(a')\big)}=\frac{\exp\big(-\eta_{k}L_{k-1}(a)\big)}{\sum_{a'}\exp\big(-\eta_{k}L_{k-1}(a')\big)}\cdot\frac{\exp\big(-\widehat{\eta}_{k}\alpha_{k}\ell_{k}(a)\big)}{\sum_{a'}\pi_{k}(a')\exp\big(-\widehat{\eta}_{k}\alpha_{k}\ell_{k}(a')\big)}\\
 & =\pi_{k}(a)\frac{\exp\big(-\widehat{\eta}_{k}\alpha_{k}\ell_{k}(a)\big)}{\sum_{a'}\pi_{k}(a')\exp\big(-\widehat{\eta}_{k}\alpha_{k}\ell_{k}(a')\big)}\ge\big[1-\widehat{\eta}_{k}\alpha_{k}\ell_{k}(a)\big]\pi_{k}(a),
\end{align*}
where the last inequality holds since $\exp(-x)\geq1-x$ and $\sum_{a}\pi_{k}(a)\exp\big(-\widehat{\eta}_{k}\alpha_{k}\ell_{k}(a)\big)\leq\sum_{a}\pi_{k}(a)=1$. 

Next, suppose that $\widehat{\eta}_{k}\alpha_{k} \|\ell_k\|_{\infty} \le 1/3$.  
In this case, it is self-evident that $\widehat{\eta}_{k}\alpha_{k} |\ell_k(a) - \mathbb{E}_{\pi_k}[\ell_k] | \le 2/3$ for all $a\in \cA$.
Recalling that $\mathbb{E}_{\pi_k}[\ell_k] = \sum_a \pi_{k}(a) \ell_k(a)$, one can derive 
\begin{align}
\pi_{k+1}^{-}(a) & =\pi_{k}(a)\frac{\exp\big(-\widehat{\eta}_{k}\alpha_{k}\ell_{k}(a)\big)}{\sum_{a'}\pi_{k}(a')\exp\big(-\widehat{\eta}_{k}\alpha_{k}\ell_{k}(a')\big)}=\frac{\exp\big(-\widehat{\eta}_{k}\alpha_{k}\big(\ell_{k}(a)-\mathbb{E}_{\pi_{k}}[\ell_{k}]\big)\big)}{\sum_{a'}\pi_{k}(a')\exp\big(-\widehat{\eta}_{k}\alpha_{k}\big(\ell_{k}(a')-\mathbb{E}_{\pi_{k}}[\ell_{k}]\big)\big)}\pi_{k}(a) \notag\\
 & \geq\frac{1-\widehat{\eta}_{k}\alpha_{k}\big(\ell_{k}(a)-\mathbb{E}_{\pi_{k}}[\ell_{k}]\big)}{\sum_{a'}\pi_{k}(a')\exp\big(-\widehat{\eta}_{k}\alpha_{k}\big(\ell_{k}(a')-\mathbb{E}_{\pi_{k}}[\ell_{k}]\big)\big)}\pi_{k}(a) \notag\\
& \ge \frac{1-\widehat{\eta}_{k}\alpha_{k}\big(\ell_{k}(a)-\mathbb{E}_{\pi_{k}}[\ell_{k}]\big)}{1+\widehat{\eta}_{k}^2\alpha_{k}^2 \mathsf{Var}_{\pi_k}(\ell_{k})}\pi_{k}(a);
	\label{eq:pi-kplus1-a-LB1}
\end{align}
here, the first inequality arises since $\exp(-x)\geq1-x$, while the second inequality can be shown via the elementary inequality
$\exp(-x)\leq1-x+x^{2}$ for any $x\geq-1.5$ and therefore
\begin{align*}
&\sum_{a}\pi_{k}(a)\exp\Big(-\widehat{\eta}_{k}\alpha_{k}\big(\ell_{k}(a)-\mathbb{E}_{\pi_{k}}[\ell_{k}]\big)\Big) \\
&\qquad \leq\sum_{a}\pi_{k}(a)\bigg\{1-\widehat{\eta}_{k}\alpha_{k}\big(\ell_{k}(a)-\mathbb{E}_{\pi_{k}}[\ell_{k}]\big)+\widehat{\eta}_{k}^2\alpha_{k}^2\big(\ell_{k}(a)-\mathbb{E}_{\pi_{k}}[\ell_{k}]\big)^{2}\bigg\}\\
&\qquad =\sum_{a}\pi_{k}(a)\bigg\{1+\widehat{\eta}_{k}^2\alpha_{k}^2\big(\ell_{k}(a)-\mathbb{E}_{\pi_{k}}[\ell_{k}]\big)^{2}\bigg\} \\
&\qquad = 1 + \widehat{\eta}_{k}^2\alpha_{k}^2 \mathsf{Var}_{\pi_k}(\ell_{k}).
\end{align*}
Applying the elementary inequality $\frac{1-a}{1+b}\geq(1-a)(1-b)=1-a-b+ab\geq1-a-2b$ for any $a\in [-1,1]$ and $b>0$, 
we can continue to lower bound \eqref{eq:pi-kplus1-a-LB1} as follows
\begin{align*}
\eqref{eq:pi-kplus1-a-LB1} & \geq\bigg\{1-\widehat{\eta}_{k}\alpha_{k}\big(\ell_{k}(a)-\mathbb{E}_{\pi_{k}}[\ell_{k}]\big)-2\widehat{\eta}_{k}^2\alpha_{k}^2\mathsf{Var}_{\pi_{k}}\big(\ell_{k}\big)\bigg\} \pi_{k}(a),
\end{align*}
thereby completing the proof.  



\section{Proofs of auxiliary lemmas and details}

\subsection{Proof of Lemma~\ref{lem:UCB}}
\label{sec:proof-lemma:UCB}

This section aims to prove Lemma~\ref{lem:UCB}, which establishes the inequality $\widehat{V}_{i,h} \ge \overline{V}_{i,h}^{\star, \pihat_{-i}}$.
In what follows, we shall proceed with an induction argument. 
The base case with step $H+1$ is trivially true, given that 
\[
	\widehat{V}_{i,H+1}= \overline{V}_{i,H+1}^{\star, \pihat_{-i}} = 0
\]
holds for any joint policy. Next, let us assume that the claim \eqref{eq:lem-UCB} is valid for step $h+1$, namely, 
\begin{align}
	\widehat{V}_{i,h+1} \ge \overline{V}_{i,h+1}^{\star, \pihat_{-i}} ,
	\label{eq:lem-UCB-h-plus-1}
\end{align}
and attempt to justify the validity of this result when $h+1$ is replaced with $h$.

This step is mainly accomplished by applying our refined theory (cf.~Theorem~\ref{thm:FTRL-refined}) for FTRL (see \eqref{eq:policy-update-exponential-FTRL}). More precisely, we claim that
\begin{align}
\max_{a_i}Q_{i,h}^{K}(s,a_i) & \leq\sum_{k=1}^{K}\alpha_{k}^{K}\Big\langle\pi_{i,h}^{k}( s),\, q_{i,h}^{k}(s,\cdot)\Big\rangle\notag\\
 & \quad+10\sqrt{\frac{c_{\alpha}\log^{3}(KA_i)}{KH}}\sum_{k=1}^{K}\alpha_{k}^{K}\mathsf{Var}_{\pi_{i,h}^{k}(s)}\Big(q_{i,h}^{k}(s,\cdot)\Big)+2\sqrt{\frac{c_{\alpha}H\log^{3}(KA_i)}{K}}
\label{eq:Q-upper-claim}
\end{align}
for any $s\in \cS$, whose proof is deferred to Appendix~\ref{sec:proof-claim-eq:Q-upper-claim}. 
Recall the construction \eqref{eq:V-max-output} of $\widehat{V}_{i,h}$. 
If $\widehat{V}_{i,h} = H-h+1$, then the claimed result $\widehat{V}_{i,h} \geq \overline{V}_{i,h}^{\star, \pihat_{-i}}$ holds trivially. It thus suffices to focus on the case where
\begin{equation}
	\widehat{V}_{i,h}(s)  =\sum_{k=1}^{K}\alpha_{k}^{K}\Big\langle\pi_{i,h}^{k}( s),\,q_{i,h}^{k}(s,\cdot)\Big\rangle+\beta_{i,h}(s). 	
	\label{eq:V-widehat-equal}
\end{equation}
In this case, recalling the definition of $\overline{V}_{i,h}^{\star,\pihat_{-i}}(s)$ in~\eqref{defi:V-star} gives
\begin{align*}
	\overline{V}_{i,h}^{\star,\pihat_{-i}}(s) &= \max_{a_i}\sum_{k = 1}^K \alpha_{k}^{K} \Big[ r_{i,h}^k(s, a_i) + \big\langle P_{i,h}^k(\cdot\mymid s, a_i), \overline{V}_{i,h+1}^{\star,\pihat_{-i}} \big\rangle \Big] \notag\\
	& \leq \max_{a_i}\sum_{k = 1}^K \alpha_{k}^{K} \Big[ r_{i,h}^k(s, a_i) + \big\langle P_{i,h}^k(\cdot\mymid s, a_i), \widehat{V}_{i,h+1}\big\rangle \Big] = \max_{a_i}Q_{i,h}^{K}(s,a_i)\notag\\
 & \leq\sum_{k=1}^{K}\alpha_{k}^{K}\Big\langle\pi_{i,h}^{k}(s),\,q_{i,h}^{k}(s,\cdot)\Big\rangle+10\sqrt{\frac{c_{\alpha}\log^{3}(KA_i)}{KH}}\sum_{k=1}^{K}\alpha_{k}^{K}\mathsf{Var}_{\pi_{i,h}^{k}(s)}\Big(q_{i,h}^{k}(s,\cdot)\Big)+2\sqrt{\frac{c_{\alpha}H\log^{3}(KA_i)}{K}} \notag\\
 & \leq\sum_{k=1}^{K}\alpha_{k}^{K}\Big\langle\pi_{i,h}^{k}(s),\,q_{i,h}^{k}(s,\cdot)\Big\rangle+\beta_{i,h}(s)=\widehat{V}_{i,h}(s)
\end{align*}
simultaneously for all $(s,h)\in \cS \times [H]$. Here, the second line follows from the induction hypothesis \eqref{eq:lem-UCB-h-plus-1} and the definition \eqref{eq:overline-Q-expansion-12} of $Q_{i,h}^{K}$, 
the third line invokes the claim~\eqref{eq:Q-upper-claim}, 
whereas the last line comes from our choice \eqref{eq:choice-bonus-terms-V} of $\beta_{i,h}$ (provided $\cb$ is large enough) and \eqref{eq:V-widehat-equal}. 
This concludes the proof, as long as \eqref{eq:Q-upper-claim} can be justified.

\subsubsection{Proof of claim~\eqref{eq:Q-upper-claim}}
\label{sec:proof-claim-eq:Q-upper-claim}

Consider any state $s\in \cS$. By virtue of the identity $Q_{i,h}^k = \sum_{j = 1}^k \alpha_j^kq_{i,h}^j$ (see \eqref{eq:overline-Q-expansion-12}), 
the policy update rule \eqref{eq:policy-update-exponential-12} (or \eqref{eq:policy-update-exponential-FTRL}) for $\pi_{i,h}^{k}( s)$ 
can essentially be viewed as the FTRL algorithm applied to the loss vectors 
$$\ell_k = -q_{i,h}^{k}(s,\cdot),\qquad k\geq 1. $$
Moreover, recalling the definition \eqref{eq:eta-k-choice} of $\eta_{k+1}$ and the definition \eqref{eq:alphak-choice} of $\alpha_k$ (with $c_{\alpha}\geq 24$), we have
\begin{align}
\bigg(\frac{\eta_{k}}{\eta_{k+1}}\bigg)^{2} & =\frac{\alpha_{k}}{\alpha_{k-1}} 
	 =\frac{k-2+c_{\alpha}\log K}{k-1+c_{\alpha}\log K}
	 \geq \frac{k-1}{k-1+c_{\alpha}\log K}
	 = 1 - \alpha_k
	 >(1-\alpha_{k})^{2} .
	\label{eq:eta-k-condition-check}
\end{align}
This property \eqref{eq:eta-k-condition-check} permits us to invoke Theorem~\ref{thm:FTRL-refined} to obtain 
\begin{align}
 & \max_{a_i\in \cA_i} Q_{i,h}^{K}(s,a_i)-\sum_{k=1}^{K}\alpha_{k}^{K}\Big\langle\pi_{i,h}^{k}(s),\,q_{i,h}^{k}(s,\cdot)\Big\rangle=\max_{a_i\in\cA_i}\left\{ \sum_{k=1}^{K}\alpha_{k}^{K}\big\langle\pi_{i,h}^{k}(s),\,\ell_{k}\big\rangle-\sum_{k=1}^{K}\alpha_{k}^{K}\ell_{k}(a_i)\right\} \notag\\
 & \quad\leq\frac{5}{3}\sum_{k=2}^{K}\alpha_{k}^{K}\frac{\eta_{k}\alpha_{k}}{1-\alpha_{k}}\mathsf{Var}_{\pi_{i,h}^{k}(s)}\Big(q_{i,h}^{k}(s,\cdot)\Big)+\frac{\log A_i}{\eta_{K+1}}+\xi_{i,h}\notag\\
 & \quad\overset{(\mathrm{i})}{\leq}\frac{5}{3}\sum_{k=2}^{K/2}\frac{\big(2c_{\alpha}\big)^{1.5}\log^{2}K}{\sqrt{kH}}\alpha_{k}^{K}\mathsf{Var}_{\pi_{i,h}^{k}(s)}\Big(q_{i,h}^{k}(s,\cdot)\Big)\notag\\
 & \qquad\quad+\frac{20}{3}\sum_{k=K/2+1}^{K}\alpha_{k}^{K}\sqrt{\frac{c_{\alpha}\log^{2}K}{KH}}\,\mathsf{Var}_{\pi_{i,h}^{k}(s)}\Big(q_{i,h}^{k}(s,\cdot)\Big)+\frac{\log A_i}{\eta_{K+1}}+\xi_{i,h},\label{eq:Q-upper-135}
\end{align}
where $\xi_{i,h}$ is defined as 
\begin{align}
	\xi_{i,h}&\defn \frac{5}{3}\alpha_{1}^{K}\eta_{2}\big\|q_{i,h}^{1}\big\|_{\infty}^{2} + \left\{ 3\sum_{k=2}^{K}\alpha_{k}^{K}\frac{\eta_{k}^{2}\alpha_{k}^{2}}{(1-\alpha_{k})^{2}}\big\|q_{i,h}^{k}\big\|_{\infty}^{3}\ind\bigg(\frac{\eta_{k}\alpha_{k}}{1-\alpha_{k}}\big\|q_{i,h}^{k}\big\|_{\infty}>\frac{1}{3}\bigg) \right\}
	+3\alpha_{1}^{K}\eta_{2}^{2}\big\|q_{i,h}^{1}\big\|_{\infty}^{3}. 
\end{align}
Here, to see why (i) holds, we make use of the facts that
\begin{subequations}
\label{eq:alphak-properties-UB-LB-135}
\begin{align}
1-\alpha_{k} & =1-\frac{c_{\alpha}\log K}{k-1+c_{\alpha}\log K}\geq\begin{cases}
1-\frac{c_{\alpha}\log K}{1+c_{\alpha}\log K}=\frac{1}{1+c_{\alpha}\log K}\geq\frac{1}{2c_{\alpha}\log K}, & \text{if }k\geq2,\\
1-\frac{c_{\alpha}\log K}{K/2+c_{\alpha}\log K}=\frac{K}{K+2c_{\alpha}\log K}\geq\frac{1}{2}, & \text{if }k \geq K/2+1,
\end{cases}\\
\eta_{k}\alpha_{k} &= \sqrt{\frac{\log K}{\alpha_{k-1}H}}\cdot\alpha_{k} \leq\sqrt{\frac{\log K}{\alpha_{k}H}}\cdot\alpha_{k}=\sqrt{\frac{\alpha_{k}\log K}{H}}\leq\sqrt{\frac{2c_{\alpha}\log^{2}K}{kH}},
\end{align}
\end{subequations}
where the first line makes use of \eqref{eq:K-general-LB-135} for large enough $c_{\mathsf{k}}$, and the second line relies on \eqref{eq:alpha-properties} in  Lemma~\ref{lem:weight}.

To proceed, let us control the terms in \eqref{eq:Q-upper-135} separately. 
\begin{itemize}

	\item We start with the first term in \eqref{eq:Q-upper-135}. 
	The elementary bound $\big\| q_{i,h}^k \big\|_{\infty}\leq H$ in Lemma~\ref{lem:range-q} taken together with \eqref{eq:alpha-properties-2} in Lemma~\ref{lem:weight} helps us derive  
\begin{align}
\sum_{k=2}^{K/2}\frac{\alpha_{k}^{K}\log^{2}K}{\sqrt{kH}}\mathsf{Var}_{\pi_{i,h}^{k}(s)}\Big(q_{i,h}^{k}(s,\cdot)\Big) & \leq\sum_{k=2}^{K/2}\frac{\log^{2}K}{K^{6}\sqrt{kH}}\mathsf{Var}_{\pi_{i,h}^{k}(s)}\Big(q_{i,h}^{k}(s,\cdot)\Big) \notag\\
 & \leq\sum_{k=2}^{K/2}\frac{\log^{2}K}{K^{6}\sqrt{kH}}\big\|q_{i,h}^{k}(s,\cdot)\big\|_{\infty}^{2}
 \leq\frac{H^{3/2}\log^{2}K}{K^{6}}\sum_{k=2}^{K/2}\frac{1}{\sqrt{k}}\notag\\
 & \leq\frac{2H^{3/2}\log^{2}K}{K^{6}}\cdot\sqrt{K/2}\le\frac{2H^{3/2}\log^{2}K}{K^{5}}.
	\label{eq:sum-var-q-plus-beta-135}
\end{align}

\item Turning to the third term in \eqref{eq:Q-upper-135}, we recall the definition of $\eta_{K+1}$ (cf.~\eqref{eq:eta-k-choice}) to obtain
\begin{align}
\frac{\log A_i}{\eta_{K+1}} 
 & =  \log A_i \sqrt{\frac{\alpha_{K}H}{\log K}} \le \sqrt{\frac{2c_{\alpha}H \log^{2}A_i}{K}},\label{eq:logA-eta-K-UB-135}
\end{align}
where the inequality comes from Lemma~\ref{lem:weight}.

\item Finally, we move on to the last term  in \eqref{eq:Q-upper-135}. 
For any $k\geq 2$, combine Lemma~\ref{lem:range-q} with \eqref{eq:alphak-properties-UB-LB-135} to obtain
\begin{align}
\frac{\eta_{k}\alpha_{k}}{1-\alpha_{k}}\big\|q_{i,h}^{k}\big\|_{\infty} 
 & \leq\frac{\sqrt{\frac{2c_{\alpha}\log^{2}K}{kH}}}{\frac{1}{2c_{\alpha}\log K}}\cdot H  =\sqrt{\frac{8c_{\alpha}^{3}H\log^{4}K}{k}}.
	\label{eq:norm-eta-q-UB-123}
\end{align}
Clearly, the right-hand side of \eqref{eq:norm-eta-q-UB-123} is upper bounded by $1/3$ for all $k$ obeying $k\geq c_9 H\log^4 \frac{K}{\delta}$ for some large enough constant $c_9>0$ (see also \eqref{eq:K-general-LB-135}).  
Consequently, one can derive
\begin{align}
\xi_{i,h} & = \frac{5}{3}\alpha_{1}^{K}\eta_{2}\big\|q_{i,h}^{1}\big\|_{\infty}^{2} + \left\{ 3\sum_{k=2}^{K}\alpha_{k}^{K}\frac{\eta_{k}^{2}\alpha_{k}^{2}}{(1-\alpha_{k})^{2}}\big\|q_{i,h}^{k}\big\|_{\infty}^{3}\ind\bigg(\frac{\eta_{k}\alpha_{k}}{1-\alpha_{k}}\big\|q_{i,h}^{k}\big\|_{\infty}>\frac{1}{3}\bigg)\right\} +3\alpha_{1}^{K}\eta_{2}^{2}\big\|q_{i,h}^{1}\big\|_{\infty}^{3}\nonumber\\
 & \leq \frac{5}{3K^{6}}\sqrt{\frac{\log K}{H}}\big\|q_{i,h}^{1}\big\|_{\infty}^{2} + \frac{\big(2\calpha\log K\big)^{2}}{K^{6}}\left\{ 3\sum_{k=2}^{c_{9}H\log^{4}\frac{K}{\delta}}\eta_{k}^{2}\alpha_{k}^{2}\big\|q_{i,h}^{k}\big\|_{\infty}^{3}\right\} +\frac{3}{K^{6}}\frac{\log K}{H}\big\|q_{i,h}^{1}\big\|_{\infty}^{3}\nonumber\\
 & \leq\frac{24c_{\alpha}^{3}\log^{4}K}{K^{6}H}\left\{ \sum_{k=1}^{K}\frac{1}{k}H^{3}\right\} \nonumber\\
 & \leq\frac{24c_{\alpha}^{3}H^{2}\log^{5}K}{K^{6}}\leq\frac{1}{K^{4}}, \label{eq:xihk-UB-145}
\end{align}
where the second line comes from \eqref{eq:alphak-properties-UB-LB-135} and the fact that $K/2 > c_9 H\log^4 \frac{K}{\delta}$ (as a consequence of \eqref{eq:K-general-LB-135}), 
and the third line holds due to Lemma~\ref{lem:range-q}. 
\end{itemize}

Putting the preceding bounds together and substituting them into \eqref{eq:Q-upper-135}, we arrive at 
\begin{align}
 & \max_{a_i}Q_{i,h}^{K}(s,a_i)-\sum_{k=1}^{K}\alpha_{k}^{K}\Big\langle\pi_{i,h}^{k}(s),\,q_{i,h}^{k}(s,\cdot)\Big\rangle\notag\nonumber \\
 & \quad\leq\frac{5(2c_{\alpha})^{1.5}}{3}\cdot \frac{2H^{3/2}\log^{2}K}{K^{5}} +\frac{20}{3}\sqrt{\frac{c_{\alpha}\log^{2}K}{KH}}\sum_{k=K/2+1}^{K}\alpha_{k}^{K}\mathsf{Var}_{\pi_{i,h}^{k}(s)}\Big(q_{i,h}^{k}(s,\cdot)\Big)+\sqrt{\frac{2c_{\alpha}H \log^{2}A_i }{K}}+\frac{1}{K^{4}}\notag\nonumber \\
 & \quad\leq  10\sqrt{\frac{c_{\alpha}\log^{3}(KA_i)}{KH}}\sum_{k=1}^{K}\alpha_{k}^{K}\mathsf{Var}_{\pi_{i,h}^{k}(s)}\Big(q_{i,h}^{k}(s,\cdot)\Big)+2\sqrt{\frac{c_{\alpha}H\log^{3}(KA_i)}{K}},\label{eq:Q-sa-sum-q-UB-135}
\end{align}
where the last line is valid under Condition~\eqref{eq:K-general-LB-135}. This completes the proof of Claim~\eqref{eq:Q-upper-claim}.

\subsection{Proof of Lemma~\ref{lem:V-upper}}
\label{sec:proof-lemma-V-upper}

In this section, we present the proof of Lemma~\ref{lem:V-upper}. To begin with, we introduce the auxiliary quantities
\begin{align*}
\widetilde{q}_{i,h}^k(s, a_i) \coloneqq r_{i,h}^k(s, a_i) + P_{i,h}^k(\cdot\mymid s, a_i)\overline{V}_{i,h+1}^{\pihat},
\qquad \forall (s,a_i)\in \cS\times\cA_i. 
\end{align*}
It is also helpful to introduce an auxiliary random action $\ak \in \cA_i$ generated in a way that
\begin{align*}
	\ak \sim \pi_{i,h}^k(s), 
\end{align*}
which is independent from $\widetilde{q}_{i,h}^k$ conditional on $\pi_{i,h}^k$.
This allows us to define another set of random variables
\begin{align}
	\widehat{q}_{i,h}^k(s) \coloneqq  \widetilde{q}_{i,h}^k\big(s, \ak) , 
	\qquad \forall s\in \cS, 
	\label{eq:defn-qh-hat}
\end{align}
which plays a central role in our analysis. 
It is readily seen from the facts $\overline{V}_{i,h+1}(s)\leq H-h$ (cf.~\eqref{eq:V-max-output}) and $r_{i,h}^k(s,a_i)\in [0,1]$ that 
\begin{align}
	0\leq \widehat{q}_{i,h}^k(s), \widetilde{q}_{i,h}^k(s, a_i) \leq H-h+1, \qquad \forall (s,a_i,h,k)\in \cS\times\cA_i\times [H]\times [K].
	\label{eq:q-hat-range}
\end{align}
Letting $e(i)\in \mathbb{R}^{A_i}$ denote the $i$-th standard basis vector,
we learn from the law of total variance that
\begin{align}
\mathsf{Var}_{h,k-1}\Big(\widehat{q}_{i,h}^{k}(s)\Big) & =\mathsf{Var}_{h,k-1}\Big(\big\langle e(\ak),\,\widetilde{q}_{i,h}^{k}(s,\cdot)\big\rangle\Big) \notag\\
	& \geq\mathsf{Var}_{h,k-1}\Big(\mathbb{E}_{h,k-1}\Big[\big\langle e(\ak),\,\widetilde{q}_{i,h}^{k}(s,\cdot)\big\rangle\mid\widetilde{q}_{i,h}^{k}\Big]\Big) \notag\\
 & =\mathsf{Var}_{h,k-1}\Big(\big\langle\pi_{i,h}^{k}(s),\,\widetilde{q}_{i,h}^{k}(s,\cdot)\big\rangle\Big).
	\label{eq:law-total-variance}
\end{align}
With these preparations in place, we are ready to embark on the proof.

\subsubsection{Proof of inequalities \eqref{eq:V-upper-12} and \eqref{eq:V-upper-34}}

Recall the definition of $\overline{V}_{i,h}^{\pihat}(s)$ in~\eqref{defi:V-muhat} that
\begin{equation}	
	\overline{V}_{i,h}^{\pihat}(s) = \sum_{k = 1}^K \alpha_{k}^{K} \mathop{\mathbb{E}}\limits_{a_i\sim \pi_{i,h}^k(s)}\Big[r_{i,h}^k(s, a_i) + P_{i,h}^k(\cdot\mymid s, a_i)\overline{V}_{i,h+1}^{\pihat}\Big] 
	= \sum_{k = 1}^K \alpha_k^K \Big\langle \pi_{i,h}^k( s), \,\widetilde{q}_{i,h}^k(s,\cdot) \Big\rangle.
	\label{eq:Vh-equation-analysis}
\end{equation}
It is first observed that
\begin{align}
\sum_{k=1}^{K}\mathbb{E}_{h,k-1}\Big[\alpha_{k}^{K}\big\langle\pi_{i,h}^{k}(s),\widetilde{q}_{i,h}^{k}(s,\cdot)\big\rangle\Big] 
	& = \sum_{k=1}^{K}\alpha_{k}^{K}\mathop{\mathbb{E}}_{ \bm{a} \sim\pi_{h}^{k}(s) }\Big[r_{i,h}(s,\bm{a})
	+ \big\langle P_{i,h}(\cdot\mymid s,\bm{a}),\overline{V}_{i,h+1}^{\pihat}\big\rangle\mid\overline{V}_{i,h+1}^{\pihat},\pi_{i,h}^{k}\Big] \nonumber \\
 & = r_{i,h}^{\pihat}(s)+\big\langle P_{i,h}^{\pihat}(s,\cdot),\overline{V}_{i,h+1}^{\pihat}\big\rangle,\label{eq:mean-sum-alpha-mu-q}
\end{align}
where the second identity arises from the definitions \eqref{eq:defn-rh-Ph-mu-nu} of $r_{i,h}^{\pihat}$ 
and $P_{i,h}^{\pihat}$. 
It is also seen that
\begin{align*}
R_{1} & \coloneqq\max_{k}\Big|\alpha_{k}^{K}\big\langle\pi_{i,h}^{k}(s),\widetilde{q}_{i,h}^{k}(s,\cdot)\big\rangle\Big|
	\leq\Big\{\max_{k}\alpha_{k}^{K}\Big\}\Big\{ \max_{k}\big\|\pi_{i,h}^{k}(s)\big\|_{1}\big\|\widetilde{q}_{i,h}^{k}\big\|_{\infty}\Big\} 
	\leq\frac{2c_{\alpha}H\log K}{K} ,
\end{align*}
where the first line invokes Lemma~\ref{lem:weight},~\eqref{eq:q-hat-range} and the fact $\|\pi_{i,h}^{k}(s)\|_{1}=1$. 
Another observation is that
%
\begin{align}
W_{1}=\sum_{k=1}^{K}\big(\alpha_{k}^{K}\big)^{2}\mathsf{Var}_{h,k-1}\Big(\big\langle\pi_{i,h}^{k}(s),\widetilde{q}_{i,h}^{k}(s,\cdot)\big\rangle\Big) & \leq\left\{ \max_{k}\alpha_{k}^{K}\right\} \left\{ \sum_{k=1}^{K}\alpha_{k}^{K}\mathsf{Var}_{h,k-1}\Big(\big\langle\pi_{i,h}^{k}(s),\widetilde{q}_{i,h}^{k}(s,\cdot)\big\rangle\Big)\right\} \notag\\
 & \leq\frac{2c_{\alpha}\log K}{K}\sum_{k=1}^{K}\alpha_{k}^{K}\mathsf{Var}_{h,k-1}\Big(\widehat{q}_{i,h}^{k}(s)\Big),  
\end{align}
where the second line makes use of Lemma~\ref{lem:weight} and the inequality \eqref{eq:law-total-variance}. With the definitions \eqref{eq:Vh-equation-analysis} and \eqref{eq:mean-sum-alpha-mu-q} in mind, invoking Freedman's inequality (i.e., Theorem~\ref{thm:Freedman}) 
with $\kappa_1= \sqrt{\frac{K\log \frac{K}{\delta}}{H}}$ then leads to
\begin{align}
& \bigg| \overline{V}_{i,h}^{\pihat}(s)  -\Big(r_{i,h}^{\pihat}(s)+\big\langle P_{i,h}^{\pihat}(s,\cdot),\,\overline{V}_{i,h+1}^{\pihat}\big\rangle\Big)\bigg| \notag \\
 &\quad\quad    =\bigg|\sum_{k=1}^{K}\alpha_{k}^{K}\Big\langle\pi_{i,h}^{k}(s),\widetilde{q}_{i,h}^{k}(s,\cdot)\Big\rangle -\sum_{k=1}^{K}\mathbb{E}_{h,k-1}\Big[\alpha_{k}^{K}\Big\langle\pi_{i,h}^{k}(s),\widetilde{q}_{i,h}^{k}(s,\cdot)\Big\rangle\Big]\bigg|\nonumber \\
 & \quad\quad\leq\kappa_{1}W_{1}+\left(\frac{2}{\kappa_{1}}+5R_{1}\right)\log\frac{3K}{\delta}\nonumber \\
 & \quad\quad\leq2c_{\alpha}\sqrt{\frac{\log^{3}\frac{K}{\delta}}{KH}}\sum_{k=1}^{K}\alpha_{k}^{K}\mathsf{Var}_{h,k-1}\Big(\widehat{q}_{i,h}^{k}(s)\Big)+\left(2\sqrt{\frac{H}{K\log\frac{K}{\delta}}}+\frac{10c_{\alpha}H\log K}{K}\right)\log\frac{3K}{\delta}\nonumber \\
	& \quad\quad\leq2c_{\alpha}\sqrt{\frac{\log^{3}\frac{K}{\delta}}{KH}}\sum_{k=1}^{K}\alpha_{k}^{K}\mathsf{Var}_{h,k-1}\Big(\widehat{q}_{i,h}^{k}(s)\Big)+4\sqrt{\frac{H\log\frac{3K}{\delta}}{K}}
	\label{eq:sum-alpha-mu-r-PV-diff}
\end{align}
with probability at least $1-\delta$, 
where the last relation holds true under Condition~\eqref{eq:K-general-LB-135}. 

To continue, we note the first term in \eqref{eq:sum-alpha-mu-r-PV-diff} can be bounded by Cauchy-Schwarz as follows:
\begin{align}
\sum_{k=1}^{K}\alpha_{k}^{K}\mathsf{Var}_{h,k-1}\Big(\widehat{q}_{i,h}^{k}(s)\Big) & =\sum_{k=1}^{K}\alpha_{k}^{K}\mathbb{E}_{h,k-1}\left[\big(\widehat{q}_{i,h}^{k}(s)\big)^{2}\right]-\sum_{k=1}^{K}\alpha_{k}^{K}\Big(\mathbb{E}_{h,k-1}\left[\widehat{q}_{i,h}^{k}(s)\right]\Big)^{2} \notag \\
 & \leq\sum_{k=1}^{K}\alpha_{k}^{K}\mathbb{E}_{h,k-1}\left[\big(\widehat{q}_{i,h}^{k}(s)\big)^{2}\right]-\bigg(\sum_{k=1}^{K}\alpha_{k}^{K}\mathbb{E}_{h,k-1}\left[\widehat{q}_{i,h}^{k}(s)\right]\bigg)^{2}. \label{eq:Vh-minus-r-Ph-bound-789}
\end{align}
Further, we make note of two additional facts: 
\begin{itemize}
	\item The weighted mean of $\widehat{q}_{i,h}^{k}(s)$ obeys
\begin{align}
\sum_{k=1}^{K}\alpha_{k}^{K}\mathbb{E}_{h,k-1}\left[\widehat{q}_{i,h}^{k}(s)\right] 
	& =\sum_{k=1}^{K}\alpha_{k}^{K}\mathop{\mathbb{E}}_{\bm{a}\sim\pi_{h}^{k}(s)}\big[r_{i,h}(s,\bm{a})\big]
	+\sum_{k=1}^{K}\alpha_{k}^{K}\mathop{\mathbb{E}}_{\bm{a} \sim\pi_{h}^{k}(s)}\Big[\big\langle P_{i,h}(\cdot\mymid s,\bm{a}),\,\overline{V}^{\pihat}_{i,h+1}\big\rangle\Big] \notag\\
 & =r_{i,h}^{\pihat}(s)+\big\langle P_{i,h}^{\pihat}(s,\cdot),\overline{V}^{\pihat}_{i,h+1}\big\rangle
	\ge  \big\langle P_{i,h}^{\pihat}(s,\cdot),\overline{V}^{\pihat}_{i,h+1}\big\rangle .
	\label{eq:weighted-q-LB-123}
\end{align}
	\item Regarding the square of $\widehat{q}_{i,h}^{k}(s)$, one has (see \eqref{eq:defn-qh-hat})
\begin{align*}
\big(\widehat{q}_{i,h}^{k}(s)\big)^{2} & =\Big(r_{i,h}^{k}(s,\ak)+\big\langle P_{i,h}^{k}(\cdot\mymid s,\ak),\,\overline{V}^{\pihat}_{i,h+1}\big\rangle\Big)^{2}\\
 & =\Big(\big\langle P_{i,h}^{k}(\cdot\mymid s,\ak),\,\overline{V}^{\pihat}_{i,h+1}\big\rangle\Big)^{2}+\Big(r_{i,h}^{k}(s,\ak)\Big)^{2}+2r_{i,h}^{k}(s,\ak)\big\langle P_{i,h}^{k}(\cdot\mymid s,\ak),\,\overline{V}^{\pihat}_{i,h+1}\big\rangle\\
 & \leq\Big(\big\langle P_{i,h}^{k}(\cdot\mymid s,\ak),\,\overline{V}^{\pihat}_{i,h+1}\big\rangle\Big)^{2}+3H\\
 & \leq\Big\langle P_{i,h}^{k}(\cdot\mymid s,\ak),\,\overline{V}^{\pihat}_{i,h+1} \circ \overline{V}^{\pihat}_{i,h+1} \Big\rangle+3H, 
\end{align*}
where we have used the fact that $\|\overline{V}^{\pihat}_{i,h+1} \|_{\infty}\leq H$ and $\| r_{i,h}^{k}\|_{\infty}\leq 1$; consequently,  
\begin{align}
\sum_{k=1}^{K}\alpha_{k}^{K}\mathbb{E}_{h,k-1}\Big[\big(\widehat{q}_{i,h}^{k}(s)\big)^{2}\Big] 
	& \le\sum_{k=1}^{K}\alpha_{k}^{K}\mathbb{E}_{h,k-1}\Big[\big\langle P_{i,h}^{k}(\cdot\mymid s,\ak),\,\overline{V}^{\pihat}_{i,h+1}\circ\overline{V}^{\pihat}_{i,h+1}\big\rangle\Big]+3H\nonumber \\
 & =\sum_{k=1}^{K}\alpha_{k}^{K}\sum_{a_i\in \cA_i}\pi^{k}_{i,h}(a_i\mymid s)\mathbb{E}_{h,k-1}\Big[\big\langle P_{i,h}^{k}(\cdot\mymid s,a_i),\,\overline{V}^{\pihat}_{i,h+1}\circ\overline{V}^{\pihat}_{i,h+1}\big\rangle\Big]+3H\nonumber \\
 & =\Big\langle P_{i,h}^{\pihat}(s,\cdot),\overline{V}^{\pihat}_{i,h+1}\circ\overline{V}^{\pihat}_{i,h+1}\Big\rangle+3H. \label{eq:weighted-qsquare-UB}
\end{align}
\end{itemize}
Taking \eqref{eq:weighted-q-LB-123} and \eqref{eq:weighted-qsquare-UB} together with  \eqref{eq:Vh-minus-r-Ph-bound-789} yields
\begin{align*}
 \sum_{k=1}^{K}\alpha_{k}^{K}\mathsf{Var}_{h,k-1}\Big(\widehat{q}_{i,h}^{k}(s)\Big) &  \leq \sum_{k=1}^{K}\alpha_{k}^{K}\mathbb{E}_{h,k-1}\left[\big(\widehat{q}_{i,h}^{k}(s)\big)^{2}\right]-\bigg(\sum_{k=1}^{K}\alpha_{k}^{K}\mathbb{E}_{h,k-1}\left[\widehat{q}_{i,h}^{k}(s)\right]\bigg)^{2}\\
	&   \leq\Big\langle P_{i,h}^{\pihat}(s,\cdot),\overline{V}^{\pihat}_{i,h+1}\circ\overline{V}^{\pihat}_{i,h+1}\Big\rangle-\Big(\big\langle P_{i,h}^{\pihat}(s,\cdot),\overline{V}^{\pihat}_{i,h+1}\big\rangle\Big)^{2}+3H .
\end{align*}
To finish up, 
substituting these into \eqref{eq:sum-alpha-mu-r-PV-diff} and making use of the assumption \eqref{eq:K-general-LB-135} give
\begin{align*}
 &\bigg|   \overline{V}^{\pihat}_{i,h}(s)-\Big(r_{i,h}^{\pihat}(s)+\big\langle P_{i,h}^{\pihat}(s,\cdot),\,\overline{V}^{\pihat}_{i,h+1}\big\rangle\Big) \bigg|  \notag\\
 & \qquad\leq2c_{\alpha}\sqrt{\frac{\log^{3}\frac{K }{\delta}}{KH}} \left[ \Big\langle P_{i,h}^{\pihat}(s,\cdot),\overline{V}^{\pihat}_{i,h+1}\circ\overline{V}^{\pihat}_{i,h+1}\Big\rangle-\Big(\big\langle P_{i,h}^{\pihat}(s,\cdot),\overline{V}^{\pihat}_{i,h+1}\big\rangle\Big)^{2} \right]  +\left(6c_{\alpha}+4\right)\sqrt{\frac{H\log^{3}\frac{K}{\delta}}{K}} 
\end{align*}
for any $s\in \cS$, 
thus concluding the proof of the first claim \eqref{eq:V-upper-12} of Lemma~\ref{lem:V-upper}.

The second claim \eqref{eq:V-upper-34} of Lemma~\ref{lem:V-upper} can be established using exactly the same argument, and hence we omit the proof here for the sake of brevity.

%

%

\subsubsection{Proof of inequality \eqref{eq:V-upper-56}}

We then turn to the last advertised inequality \eqref{eq:V-upper-56}. 
Given that
%
$\overline{r}_{i,h}(s)+\overline{P}_{i,h}(s,\cdot)\widehat{V}_{i,h+1}\in[0,H-h+1]$
for all $s\in\mathcal{S}$, we can recall the definition (\ref{eq:V-max-output})
of $\widehat{V}_{i,h}$ to obtain 
\begin{equation}
	\Big|\widehat{V}_{i,h}(s)-\big(\overline{r}_{i,h}(s)+\overline{P}_{i,h}(s,\cdot)\widehat{V}_{i,h+1}\big)\Big|\leq\bigg|\sum_{k=1}^{K}\alpha_{k}^{K}\Big\langle\pi_{i,h}^{k}(\cdot\mymid s),\,q_{i,h}^{k}(s,\cdot)\Big\rangle+\beta_{i,h}(s)-\big(\overline{r}_{i,h}(s)+\overline{P}_{i,h}(s,\cdot)\widehat{V}_{i,h+1}\big)\bigg|
	\label{eq:overline-Vh-diff-r-P-V}
\end{equation}
for all $s\in \cS$. The remaining analysis is dedicated to bounding the right-hand side of \eqref{eq:overline-Vh-diff-r-P-V}.

Let us begin with the following identity: 
\begin{align}
\sum_{k=1}^{K}\alpha_{k}^{K}\Big\langle\pi_{i,h}^{k}(\cdot\mymid s),\,q_{i,h}^{k}(s,\cdot)\Big\rangle+ \beta_{i,h}(s) & =\sum_{k=1}^{K}\alpha_{k}^{K}\mathop{\mathbb{E}}\limits _{a_i\sim\pi_{i,h}^{k}(s)}\Big[r_{i,h}^{k}(s,a_i)+P_{i,h}^{k}(\cdot\mymid s,a_i)\widehat{V}_{i,h+1}\Big]+\beta_{i,h}(s)\notag\\
	& = \overline{r}_{i,h}(s)+\big\langle \overline{P}_{i,h}(s,\cdot),\widehat{V}_{i,h+1}\big\rangle+\beta_{i,h}(s),
	\label{eq:identity-sum-mu-V}
\end{align}
where we recall the definitions of $\overline{r}_{i,h} \in \mathbb{R}^{S}$ and $\overline{P}_{i,h}\in \mathbb{R}^{S\times S}$ in \eqref{eq:defn-rh-Ph-mu-nu}. 
The key step boils down to bounding the bonus term defined in \eqref{eq:choice-bonus-terms-V}, towards which we first claim that 
\begin{align}
\sum_{k=1}^{K}\alpha_{k}^{K}\mathsf{Var}_{\pi_{i,h}^{k}(s)}\big(q_{i,h}^{k}(s,\cdot)\big) & \leq 2 + 2\Big[\overline{P}_{i,h}(s,\cdot)\big(\widehat{V}_{i,h+1} \circ \widehat{V}_{i,h+1}\big)-\big(\overline{P}_{i,h}(s,\cdot)\widehat{V}_{i,h+1}\big) ^2 \Big]
 \label{eq:variance-expectation}
\end{align}
holds for all $s \in \cS$. Assuming the validity of this claim, we can then demonstrate that
\begin{align}
	\beta_{i,h}(s) & =\cb\sqrt{\frac{\log^{3}\big( \frac{KS\sum_i A_i}{\delta} \big) }{KH}}\sum_{k=1}^{K}\alpha_{k}^{K}\Big\{\mathsf{Var}_{\pi_{i,h}^{k}(s)}\Big(q_{i,h}^{k}(s,\cdot)\Big)+H\Big\}\notag\\
	& \leq2\cb\sqrt{\frac{\log^{3}\big( \frac{KS\sum_i A_i}{\delta} \big) }{KH}}\Big\{\overline{P}_{i,h}(s,\cdot)\big(\widehat{V}_{i,h+1} \circ \widehat{V}_{i,h+1}\big)-\big(\overline{P}_{i,h}(s,\cdot)\widehat{V}_{i,h+1}\big) ^2 +H\Big\}, 
\label{eq:beta-hV-UB-147}
\end{align}
where we have used the identity $\sum_{k=1}^K\alpha_k^K =1$. 
Hence, we can readily establish the desired result \eqref{eq:V-upper-56} by combining \eqref{eq:beta-hV-UB-147} with \eqref{eq:identity-sum-mu-V} and \eqref{eq:overline-Vh-diff-r-P-V}, 
provided that $c_3>0$ is sufficiently large.

It remains to justify the claim \eqref{eq:variance-expectation}. Towards this end, we make the observation that
\begin{align*}
\mathsf{Var}_{\pi_{i,h}^{k}(s)}\big(q_{i,h}^{k}(s,\cdot)\big) & \le2\mathsf{Var}_{\pi_{i,h}^{k}(s)}\big(r_{i,h}^{k}(s,\cdot)\big)+2\mathsf{Var}_{\pi_{i,h}^{k}(s)}\Big(\sum_{s'} P_{i,h}^{k}(s'\mid s,\cdot)\widehat{V}_{i,h+1}(s')\Big)\\
 & \le2+2\left[\sum_{a_i}\pi_{i,h}^{k}(a_i\mymid s) P_{i,h}^{k}(\cdot\mymid s,a_i)\big(\widehat{V}_{i,h+1}\circ\widehat{V}_{i,h+1}\big)-\bigg(\sum_{a_i}\pi_{i,h}^{k}(a_i\mymid s) P_{i,h}^{k}(\cdot\mymid s,a_i)\widehat{V}_{i,h+1}\bigg)^{2}\right],
\end{align*}
which results from $\|r_{i,h}^{k}\|_{\infty}\leq 1$ and the following relation:  
\begin{align*}
&\mathsf{Var}_{\pi_{i,h}^{k}(s)}\Big(\sum_{s'} P_{i,h}^{k}(s'\mid s,\cdot)\widehat{V}_{i,h+1}(s')\Big) \\
	&\qquad=\sum_{a_i}\pi_{i,h}^{k}(a_i\mymid s)\Big(P_{i,h}^{k}(\cdot\mymid s,a_i)\widehat{V}_{i,h+1}\Big)^{2}-\bigg(\sum_{a_i}\pi_{i,h}^{k}(a_i\mymid s)P_{i,h}^{k}(\cdot\mymid s,a_i)\widehat{V}_{i,h+1}\bigg)^{2} \\
&\qquad\le \sum_{a_i}\pi_{i,h}^{k}(a_i\mymid s)P_{i,h}^{k}(\cdot\mymid s,a_i)\big(\widehat{V}_{i,h+1} \circ \widehat{V}_{i,h+1}\big)-\bigg(\sum_{a_i}\pi_{i,h}^{k}(a_i\mymid s)P_{i,h}^{k}(\cdot\mymid s,a_i)\widehat{V}_{i,h+1}\bigg)^{2}.
\end{align*}
This taken together with the fact $\sum_{k=1}^K\alpha_k^K =1$ and Jensen's inequality yields 
\begin{align*}
&\sum_{k=1}^{K}\alpha_{k}^{K}\mathsf{Var}_{\pi_{i,h}^{k}(s)}\big(q_{i,h}^{k}(s,\cdot)\big) \\
&\qquad \leq\sum_{k=1}^{K}\alpha_{k}^{K}\left\{ 2+2\left[\sum_{a_i}\pi_{i,h}^{k}(a_i\mymid s)P_{i,h}^{k}(\cdot\mymid s,a_i)\big(\widehat{V}_{i,h+1}\circ\widehat{V}_{i,h+1}\big)-\bigg(\sum_{a_i}\pi_{i,h}^{k}(a_i\mymid s)P_{i,h}^{k}(\cdot\mymid s,a_i)\widehat{V}_{i,h+1}\bigg)^{2}\right]\right\} \\
 &\qquad \leq2+2\overline{P}_{i,h}(s,\cdot)\big(\widehat{V}_{i,h+1}\circ\widehat{V}_{i,h+1}\big)-2\bigg(\sum_{k=1}^{K}\alpha_{k}^{K}\sum_{a_i}\pi_{i,h}^{k}(a_i\mymid s)P_{i,h}^{k}(\cdot\mymid s,a_i)\widehat{V}_{i,h+1}\bigg)^{2}\\
 &\qquad =2+2\Big[\overline{P}_{i,h}(s,\cdot)\big(\widehat{V}_{i,h+1}\circ\widehat{V}_{i,h+1}\big)-\big(\overline{P}_{i,h}(s,\cdot)\widehat{V}_{i,h+1}\big)^{2}\Big]
\end{align*}
as claimed.

\subsection{Minimax lower bound}
\label{sec:lower}

In this section, we formalize the minimax lower bound claimed in \eqref{eq:minimax-lower-sample-complexity}.

\begin{theorem}[Minimax lower bound]
\label{thm:lower} 
Consider any $m\geq 2$ and any $0<\varepsilon \leq c_1 H$ for some small enough constant $c_1>0$. 
Then one can construct a collection of $m$-player zero-sum Markov games $\{\mathcal{MG}_{\theta} \mid \theta \in \Theta\}$ with $S$ states, horizon $H$, and $A_i$ actions for the $i$-th player ($1\leq i\leq m$) such that
\begin{align}
	\inf_{\widehat{\pi}} \max_{\theta\in\Theta} \mathbb{P}^{\mathcal{MG}_{\theta}}\left\{ \NEgap\big(\widehat{\pi}\big) > \varepsilon\right\} \geq \frac{1}{4},
\end{align}
provided that the total sample size obeys
\begin{align}
	N \leq \frac{c_2 H^4S \max_{1\leq i\leq m}A_i }{\varepsilon^2} \label{eq:sample-size-lower}
\end{align}
	for some sufficiently small constant $c_2>0$. Here, the infimum is over all (joint) policy estimator $\widehat{\pi}$, and $\mathbb{P}^{\mathcal{MG}_{\theta}}$ denotes the probability when the Markov game is $\mathcal{MG}_{\theta}$. 
\end{theorem}
%


\begin{proof}
	Suppose without loss of generality that $A_1\geq \max\{A_2,\ldots,A_m\}$. 
Let us begin by considering the special scenario with $A_2=\ldots=A_m=1$;  
in this case, computing either the NE or the CCE 
reduces to finding the optimal policy of a single-agent MDP with $S$ states and $A_1$ actions. 
It is well-known that for any given accuracy level $\varepsilon \in (0,H]$, there exists a non-stationary MDP with $S$ states and $A_1$ actions such that no algorithm can learn an $\varepsilon$-optimal policy 
with ${o}\big( \frac{H^4SA_1}{\varepsilon^2} \big) $ samples \citep{azar2013minimax,li2022settling}. 
More precisely, for any given $0<\varepsilon \leq c_1 H$ for some small enough constant $c_1>0$, 
one can construct a collection of MDPs $\{\mathcal{M}_{\theta} \mid \theta \in \Theta\}$ such that
\begin{align}
	\inf_{\widehat{\mu}} \max_{\theta\in\Theta} \mathbb{P}^{\mathcal{M}_{\theta}}\left\{ \max_{s\in \cS} \big(V_1^{\star}(s) - V_1^{\widehat{\mu}}(s)\big) > \varepsilon\right\} \geq \frac{1}{4},
\end{align}
with the proviso that the total sample size obeys
\begin{align}
	N \leq \frac{c_2 H^4SA_1}{\varepsilon^2}
\end{align}
for some small enough constant $c_2>0$. 
Here, the infimum is over all policy estimate $\widehat{\mu}$ in this single-agent scenario, and $\mathbb{P}^{\mathcal{M}_{\theta}}$ denotes the probability when the MDP is ${\mathcal{M}}_{\theta}$. 

Next, let us construct a collection of Markov games by augmenting each of the single-agent MDPs $\mathcal{M}_{\theta}$ 
with $A_i$ completely identical actions for the $i$-th player ($2\leq i\leq m$); that is, to construct $\mathcal{MG}_{\theta}$, we take its reward function and probability transition kernel to be
\begin{align}
	r_{i,h}^{\mathcal{MG}_{\theta}}(s, \bm{a}) = \begin{cases} r_h^{{\mathcal{M}}_{\theta}}(s, a_1) & \text{if }i=1 \\
		 -r_h^{{\mathcal{M}}_{\theta}}(s, a_1) & \text{if } i = m \\
		 0 & \text{else}
	\end{cases}
	\qquad\text{and}\qquad P_h^{\mathcal{MG}_{\theta}}(\cdot \mymid s, \bm{a}) = P_h^{{\mathcal{M}}_{\theta}}(\cdot \mymid s, a_1)
\end{align}
	for all $(s, h, \bm{a}=[a_1,\ldots,a_m]) \in \cS\times [H] \times \cA$.
	Evidently, finding either an NE or a CCE of $\mathcal{MG}_{\theta}$ is equivalent to computing the optimal policy of ${\mathcal{M}}_{\theta}$, given the non-distinguishability of the actions of all but the first player in $\mathcal{MG}_{\theta}$. 
	This in turn immediately establishes the advertised lower bound. 
\end{proof}

\subsection{Freedman's inequality}
\label{sec:Freedman}

In this section, we record the Freedman inequality for martingales \citep{freedman1975tail} with slight modification, 
which is a crucial concentration bound for our analysis.
\begin{theorem}\label{thm:Freedman}
Suppose that $Y_{n}=\sum_{k=1}^{n}X_{k}\in\mathbb{R}$,
where $\{X_{k}\}$ is a real-valued scalar sequence obeying 
\[
\left|X_{k}\right|\leq R\qquad\text{and}\qquad\mathbb{E}\left[X_{k}\mid\left\{ X_{j}\right\} _{j:j<k}\right]=0\quad\quad\quad\text{for all }k\geq1 
\]
for some quantity $R>0$. Define
\[
W_{n}\coloneqq\sum_{k=1}^{n}\mathbb{E}_{k-1}\left[X_{k}^{2}\right],
\]
where $\mathbb{E}_{k-1}$ stands for the expectation conditional
on $\left\{ X_{j}\right\} _{j:j<k}$.   
Consider any arbitrary quantity $\kappa > 0$. With probability at least $1-\delta$, one has
\begin{align}
\left|Y_{n}\right|&\leq \sqrt{8W_{n}\log\frac{3 n}{\delta}}+5R\log\frac{3 n}{\delta} \leq \kappa W_{n} + \Big(\frac{2}{\kappa} + 5R\Big)\log\frac{3 n}{\delta}. \label{eq:Freedman-random}
\end{align}
\end{theorem}

\begin{proof}
Suppose that $W_{n}\leq\sigma^{2}$ holds deterministically for some quantity $\sigma^2$.  
As has been demonstrated in \citet[Theorem 5]{li2021q}, with probability at least $1-\delta$ we have
\begin{align}
	\left|Y_{n}\right|&\leq \sqrt{8\max \bigg\{ W_{n}, \frac{\sigma^2}{2^K} \bigg\} \log\frac{2K}{\delta}}+\frac{4}{3}R\log\frac{2K}{\delta} 
\end{align}
for any positive integer $K\geq 1$. 
Recognizing the trivial bound $W_n\leq nR^2$, one can take 
$\sigma^2 = n R^2$ and $K=\log_2 n$ to obtain
\begin{align*}
\left|Y_{n}\right| & \leq\sqrt{8\max\big\{ W_{n},R^{2}\big\}\log\frac{4\log_{2}n}{\delta}}+\frac{4}{3}R\log\frac{4\log_{2}n}{\delta}\\
 & \leq\sqrt{8W_{n}\log\frac{3n}{\delta}}+\sqrt{8R^{2}\log\frac{3n}{\delta}}+\frac{4}{3}R\log\frac{3n}{\delta}\\
 & \leq\sqrt{8W_{n}\log\frac{3n}{\delta}}+5R\log\frac{3n}{\delta},
\end{align*}
where we have used $4\log_2n \leq 3n$ for any integer $n\geq 1$. This establishes the first inequality in \eqref{eq:Freedman-random}. 
The second inequality in \eqref{eq:Freedman-random} is then a direct consequence of the elementary inequality $2ab\leq a^2+b^2$.   
\end{proof}



\bibliography{bibfileRL,bibfileGame}
\bibliographystyle{apalike}

\end{document}